\documentclass[a4paper,11pt]{article}

\usepackage[margin=1in,footskip=0.5in]{geometry}
\usepackage[utf8]{inputenc} 
\usepackage[T1]{fontenc}    
\usepackage{hyperref}       
\usepackage{url}            
\usepackage{booktabs}       
\usepackage{amsfonts}       
\usepackage{nicefrac}       
\usepackage{microtype}      
\usepackage{xcolor}         
\usepackage{import}

\title{Approximation Bounds for Recurrent Neural Networks with Application to Regression}
\author{
Yuling Jiao\thanks{School of Artificial Intelligence and Hubei Key Laboratory of Computational Science, Wuhan University, Wuhan, China. Email: yulingjiaomath@whu.edu.cn}
\and
Yang Wang\thanks{Department of Mathematics, University of Hong Kong, Hong Kong, China. Email: yang.wang@hku.hk}
\and
Bokai Yan\thanks{Department of Mathematics, The Hong Kong University of Science and Technology, Hong Kong, China. Email: byanac@connect.ust.hk}
}
\date{}

\begin{document}
\maketitle

\begin{abstract}
We study the approximation capacity of deep ReLU recurrent neural networks (RNNs) and explore the convergence properties of nonparametric least squares regression using RNNs. We derive upper bounds on the approximation error of RNNs for H\"older smooth functions, in the sense that the output at each time step of an RNN can approximate a H\"older function that depends only on past and current information, termed a past-dependent function. This allows a carefully constructed RNN to simultaneously approximate a sequence of past-dependent H\"older functions. We apply these approximation results to derive non-asymptotic upper bounds for the prediction error of the empirical risk minimizer in regression problem. Our error bounds achieve minimax optimal rate under both exponentially $\beta$-mixing and i.i.d. data assumptions, improving upon existing ones. Our results provide statistical guarantees on the performance of RNNs.
\end{abstract}

\section{Introduction}

Recurrent neural networks (RNNs; \cite{rumelhart1986learning}) have attracted much attention in machine learning and artificial intelligence communities in the past few decades. As one of the earliest attempts, RNNs are designed to capture temporal dependencies and process sequential data by retaining information from previous inputs in the current hidden state. RNNs have achieved remarkable successes in many machine learning tasks such as machine translation \cite{cho2014learning}, stock price prediction \cite{selvin2017stock}, weather forecasting \cite{salman2015weather}, and speech recognition \cite{graves2014towards, bahdanau2016end}. However, theoretical explanations for their empirical success are not well established. Many fundamental questions on the theory and training dynamics of RNNs remain to be answered.

One of the fundamental questions concerning RNNs is their approximation capacity, or expressiveness: what kinds of sequence-to-sequence relationships can RNNs model? As early as the 1990s, the universal approximation theorem for shallow feedforward neural networks (FNNs) was widely discussed \cite{cybenko1989approximation, hornik1991approximation, pinkus1999approximation}, stating that a neural network with a single hidden layer can approximate any continuous function to any specified precision. Universal approximation results for shallow RNNs were subsequently proposed, for example, by \cite{matthews1993approximating, Doya1993UniversalityOF, schafer2007recurrent} in a discrete-time setting and \cite{funahashi1993approximation, chow2000modeling, li2005approximation, maass2007computational, nakamura2009approximation} in a continuous-time setting. These early studies often focused on scenarios where the target relationship is generated by some underlying dynamical system, typically in the form of difference or differential equations. The recent advancements in deep learning have sparked considerable research into the approximation theory of deep neural networks. The approximation rates of deep ReLU FNNs have been well studied for various function classes, such as continuous functions \cite{yarotsky2017error, yarotsky2018optimal, shen2020deep}, smooth functions \cite{yarotsky2020phase, lu2021deep}, piecewise smooth functions \cite{petersen2018optimal}, and shift-invariant spaces \cite{yang2022approximation}. In the context of RNNs, \cite{li2022approximation} demonstrated that temporal relationships can be effectively approximated by linear continuous-time dynamics, and \cite{hoon2023minimal} proved the universal approximation theorem for deep narrow RNNs. However, the approximation rates and the precise sense in which they apply to general RNNs remain unclear.

Another fundamental question concerns the effectiveness of RNNs: can RNNs perform well in statistical and machine learning problems, and more specifically, can RNN-based estimators achieve the optimal rate of convergence? In recent years, Transformers \cite{vaswani2017attention} have gained significant popularity and now dominate many areas in sequence modeling due to their superior experimental performance and ease of parallel training. It is often argued that Transformers outperform RNNs because of the perceived limitations of RNNs in modeling long-term dependencies \cite{bengio1994learning, hochreiter2001gradient, wen2024rnns}. Although recent variants of RNNs have shown effectiveness in long-sequence modeling through experimental results \cite{lai2018modeling, qin2023hierarchically, lin2023segrnn}, these discussions are primarily based on empirical evidence and rarely grounded in rigorous mathematical or statistical analysis. An important problem in statistics and machine learning is nonparametric regression, which seeks to estimate an unknown target regression function $f^*$ from finite observations. Recent studies have provided valuable insights into the convergence properties of nonparametric regression estimators based on deep FNNs \cite{suzuki2019adaptivity, bauer2019deep, schmidt2019deep, schmidt2020nonparametric, nakada2020adaptive, farrell2021deep, chen2022nonparametric, kohler2022estimation}. However, these analyses and convergence bounds typically assume that the observations are independently and identically distributed (i.i.d.). In many machine learning applications, this assumption does not hold. For example, in time series prediction and natural language processing, observations often exhibit temporal dependence, making the i.i.d. assumption too restrictive. Thus, it is essential to explore scenarios where this condition is relaxed. Various relaxations of the i.i.d. assumption have been proposed in the machine learning and statistics literature. A frequently used alternative is to assume that observations are drawn from a stationary mixing distribution, where the dependence between observations weakens over time. This assumption has become standard and is widely adopted in previous studies \cite{yu1994rates, mohri2008rademacher, steinwart2009fast, mohri2010stability, ralaivola2010chromatic, agarwal2012generalization, shalizi2013predictive, alquier2013prediction, kuznetsov2017generalization, ren2024statistical}. Under certain mixing conditions, carefully constructed wavelet-based and kernel-based estimators have been shown to achieve the optimal minimax rate \cite{yu1993density, viennet1997inequalities}. It remains unclear whether RNNs, which are specifically designed to handle sequential data with temporal dependencies, perform well in regression problem with dependent data.


In this paper, we study the approximation ability of recurrent neural networks and their performance in nonparametric regression.

\subsection{Contributions}

The main contributions of this work are to pose and address the following questions:

\begin{enumerate}
\item[Q1.] \textit{What kinds of sequence-to-sequence mappings can recurrent neural networks approximate, and how large must the network be to achieve a given approximation accuracy?}

\item[Q2.] \textit{In the context of nonparametric regression, can an RNN-based estimator attain the optimal rate of convergence?}
\end{enumerate}

We provide a comprehensive discussion of our setting and theoretical results in Section \ref{sec: 2}, and here only touch on some of the main aspects.



\begin{itemize}
\item We theoretically prove the equivalence between RNNs and FNNs. Specifically, we show that any RNN can be represented by a slightly larger FNN, and conversely, any FNN can be represented by a slightly larger RNN, illustrating that the function classes of RNNs and FNNs are comparable in size.

\item We propose, to the best of our knowledge, the first approximation rate for deep ReLU RNNs. We show that the output at each time step of an RNN can approximate a H\"older function that depends only on past and current tokens, referred to as a past-dependent function, and a carefully constructed RNN can simultaneously approximate a sequence of such past-dependent H\"older functions. Specifically, given a sequence of functions $\{f^{(t)}\}_{t=1}^N$ where each $f^{(t)}: [0,1]^{d_x \times t} \rightarrow \mathbb{R}^{d_y}$ depends on the $t$-prefix of the input sequence $X$ and is H\"older continuous with smoothness index $\gamma$, there exists a recurrent neural network $\mathcal{N}: \mathbb{R}^{d_x \times N} \rightarrow \mathbb{R}^{d_y \times N}$ described by width $W \asymp J \log J$ and depth $L \asymp I \log I$ for some $I, J \in \mathbb{N}$, such that for each time step $t \in \{1, \ldots, N\}$,
\begin{align*}
\sup_{X \in [0,1]^{d_x \times N}} \|\mathcal{N}(X)[t] - f^{(t)}(x[1:t])\|_{\infty} \lesssim (J I)^{-2 \gamma / (d_x t)}.
\end{align*}
We also give lower bounds on the best possible approximation error achievable by RNNs with such width and depth for H\"older smooth functions. The upper and lower bounds match up to constants and logarithmic factors, demonstrating the tightness of our bounds and the effectiveness of RNNs in approximating past-dependent relationship. These approximation results are of independent interest and may be useful in other problems.

\item We present a comprehensive error analysis of the nonparametric regression problem with weakly dependent data using RNNs. Specifically, to estimate an unknown H\"older continuous target function $f^*: [0,1]^{d_x \times N} \rightarrow \mathbb{R}$ with smoothness index $\gamma$, given an exponentially $\beta$-mixing sequence of $n$ observation pairs, we show that the RNN-based empirical risk minimizer converges in expectation at a rate of $n^{-\frac{2 \gamma}{d_x N+2 \gamma}}$ up to logarithmic factors, matching the optimal minimax rate established by \cite{stone1982optimal, viennet1997inequalities}. For algebraically $\beta$-mixing and i.i.d. settings, we obtain convergence rates of $n^{-\frac{2 r \gamma}{(r+1) d_x N + (2r+4) \gamma}}$ and $n^{-\frac{2 \gamma}{d_x N + 2 \gamma}}$, respectively, where $r$ depends on the polynomial decay rate of the $\beta$-mixing coefficients. These results are derived from a new error decomposition that bounds the excess risk by the sum of the approximation error, generalization error and dependence error. This decomposition extends the classical result for the i.i.d. data to the $\beta$-mixing setting.


\end{itemize}

\subsection{Organization}
The rest of the paper is organized as follows. In Section \ref{sec: 2}, we introduce the ReLU recurrent neural networks and the setup of the regression problem, and present our main results. Section \ref{sec: 3} is devoted to proving Proposition \ref{proposition: 3}, which establishes the equivalence between RNNs and FNNs. In Section \ref{sec: 4}, we prove Theorem \ref{theorem: 1}, which provides a novel approximation error bound for deep RNNs. In Section \ref{sec: 5}, we prove Theorem \ref{theorem: 4} by establishing a new oracle inequality involving dependent data. Finally, concluding remarks are given in Section \ref{sec: 6}.

\section{Summary of Results}\label{sec: 2}

In this section, we present our main results. For two sequences $\{a_n\}$ and $\{b_n\}$, we use the notation $a_n \lesssim b_n$ and $a_n \gtrsim b_n$ to indicate $a_n \leq C b_n$ and $a_n \geq C b_n$, respectively, for some constant $C > 0$ independent of $n$. Additionally, $a_n \asymp b_n$ signifies that both $a_n \lesssim b_n$ and $a_n \gtrsim b_n$ hold. For $\alpha \in \mathbb{R}$, the floor and ceiling functions are denoted by $\lfloor \alpha \rfloor$ and $\lceil \alpha \rceil$, which round $\alpha$ down to the nearest integer and up to the nearest integer, respectively. The set of positive integers is denoted by $\mathbb{N}=\{1,2, \ldots\}$. We also denote $\mathbb{N}_0=\mathbb{N} \cup\{0\}$ for convenience. For a set $\mathcal{S}$, $|\mathcal{S}|$ represents its cardinality. We write $\|f\|_{L^p(\lambda)} = (\int |f|^p \dd \lambda)^{1/p}$ for $p \in [1, \infty)$, where $\lambda$ is a measure on the domain of $f$, and $\|f\|_{L^\infty(\Omega)} = \sup_{x \in \Omega} |f(x)|$ for a domain $\Omega$.

\subsection{Recurrent neural networks}

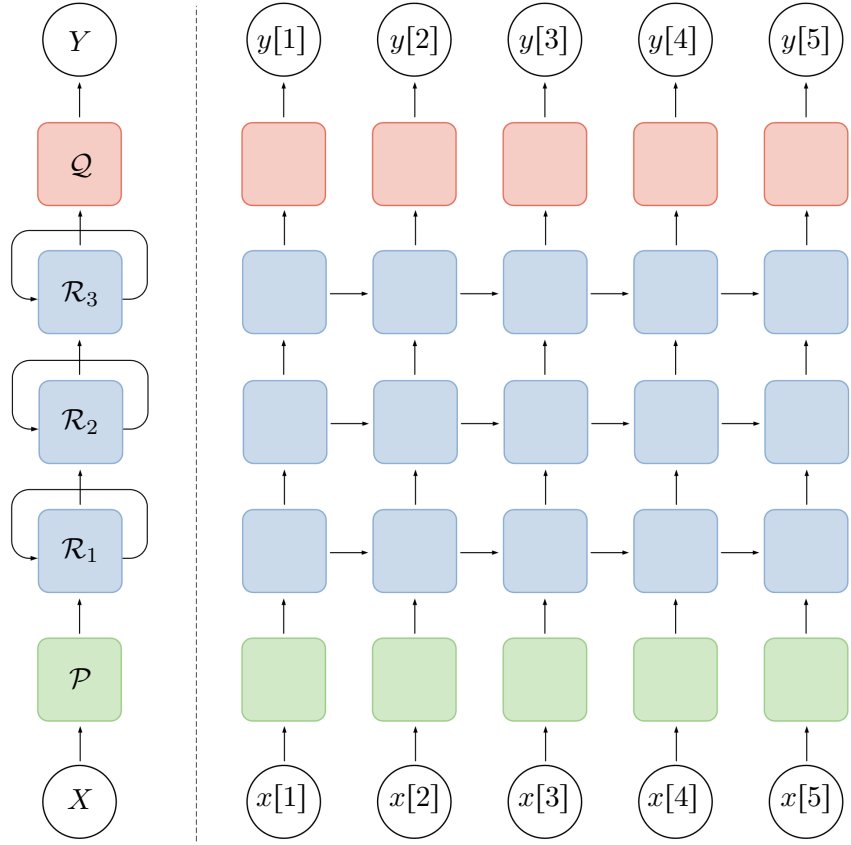
\begin{figure}[t]
\centering
\begin{tikzpicture}
    \node (image) at (0,0) {\includegraphics[width=0.7\linewidth]{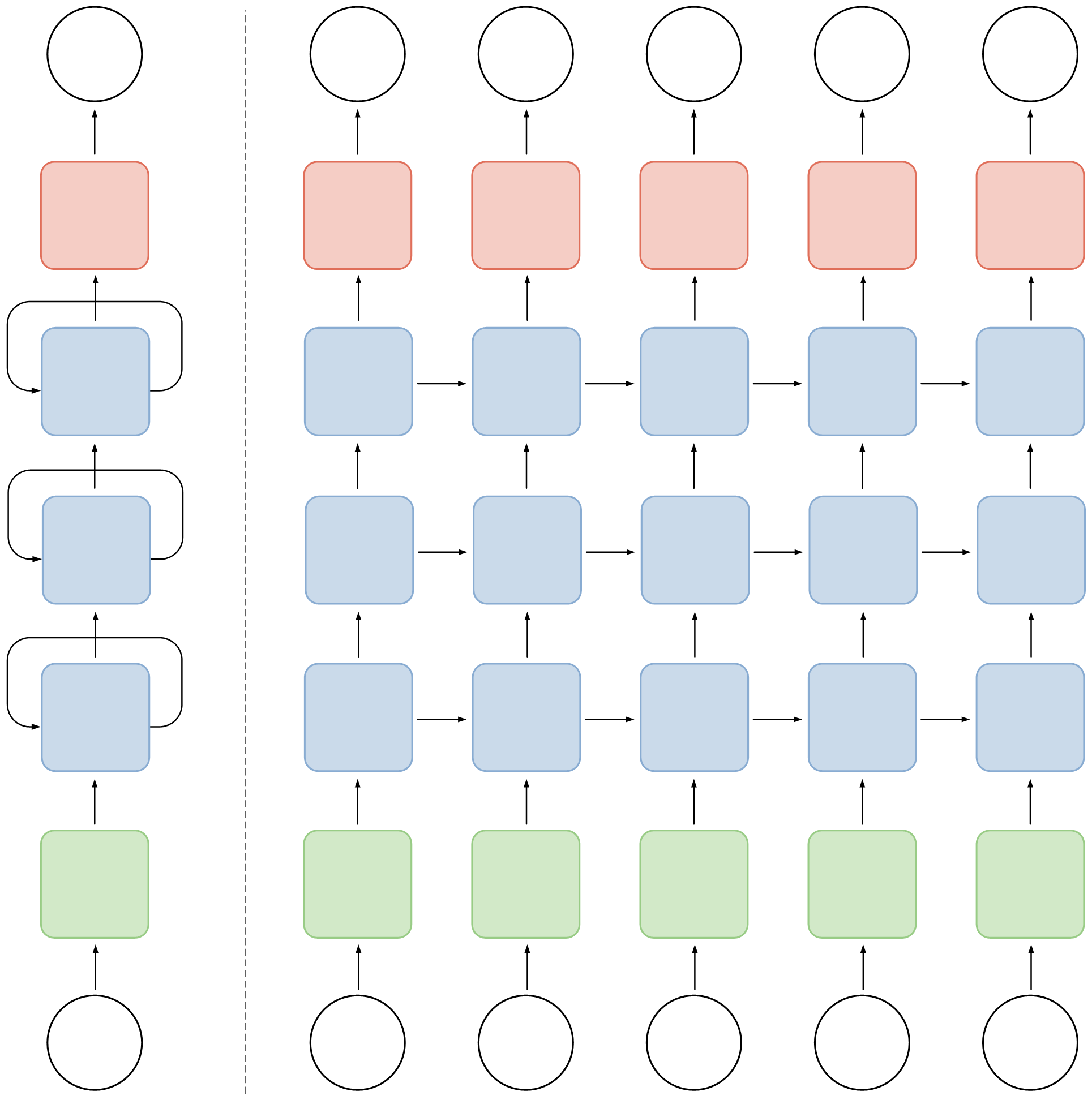}};
    \node at (-4.6, -5) {$X$};
    \node at (-4.6, -3.4) {$\mathcal{P}$};
    \node at (-4.6, -1.7) {$\mathcal{R}_1$};
    \node at (-4.6, 0) {$\mathcal{R}_2$};
    \node at (-4.6, 1.7) {$\mathcal{R}_3$};
    \node at (-4.6, 3.4) {$\mathcal{Q}$};
    \node at (-4.6, 5.05) {$Y$};

    \node at (-1.95, -5) {$x[1]$};
    \node at (-0.2, -5) {$x[2]$};
    \node at (1.5, -5) {$x[3]$};
    \node at (3.2, -5) {$x[4]$};
    \node at (4.95, -5) {$x[5]$};

    \node at (-1.95, 5.05) {$y[1]$};
    \node at (-0.2, 5.05) {$y[2]$};
    \node at (1.5, 5.05) {$y[3]$};
    \node at (3.2, 5.05) {$y[4]$};
    \node at (4.95, 5.05) {$y[5]$};
\end{tikzpicture}
\caption{An illustration of the network architecture of $Y = \mathcal{N}(X) = \mathcal{Q} \circ \mathcal{R}_3 \circ \mathcal{R}_2 \circ \mathcal{R}_1 \circ \mathcal{P}(X)$, where $\mathcal{P}$ denotes the embedding map, $\mathcal{R}_1$, $\mathcal{R}_2$, and $\mathcal{R}_3$ are the recurrent layers, and $\mathcal{Q}$ represents the projection map. In this network, the depth is $L = 3$ and the length of the input sequence is $N=5$.}
\label{figure: 1}
\end{figure}

This subsection introduces the definition of network architecture. Let $d_x$ and $d_y$ represent the dimensions of the input and output spaces, respectively, which depend only on the underlying problem. For a sequence of length $N$, $X = (x[1], x[2], \ldots, x[N])$, we refer to the sequential index $t$ as time step and each $x[t]$ as a token. The notation $x[t]_j$ denotes the $j$-th component of the $t$-th token. We use continuous indices to denote parts of the original sequence or vector, for instance, $x[a:b] = (x[a], \ldots, x[b])$ or $x[t]_{a:b} = (x[t]_a, x[t]_{a+1}, \ldots, x[t]_b)^\top \in \mathbb{R}^{b-a+1}$. In general, an RNN maps a sequence of length $N$, $X = (x[1], x[2], \ldots, x[N])$ with each $x[t] \in \mathbb{R}^{d_x}$, to another sequence of the same length, $Y = (y[1], y[2], \ldots, y[N])$ with each $y[t] \in \mathbb{R}^{d_y}$, through a series of operations. We proceed to define the necessary components.

We define the token-wise linear maps $\mathcal{P}: \mathbb{R}^{d_x \times N} \rightarrow \mathbb{R}^{W \times N}$ and $\mathcal{Q}: \mathbb{R}^{W \times N} \rightarrow \mathbb{R}^{d_y \times N}$ to connect the input, hidden state, and output space. For a given matrix $P \in \mathbb{R}^{W \times d_x}$, an input embedding map $\mathcal{P}(X)[t] := Px[t]$ transforms the input vector to the hidden state space. Similarly, for a given matrix $Q \in \mathbb{R}^{d_y \times W}$, a projection map $\mathcal{Q}(Y)[t] := Qy[t]$ projects a hidden state onto the output space.

A recurrent layer $\mathcal{R}$ maps an input sequence $X=(x[1], x[2], \ldots, x[N]) \in \mathbb{R}^{W \times N}$ to an output sequence $Y=(y[1], \ldots, y[N]) \in \mathbb{R}^{W \times N}$ by operating within the hidden state space. Here, the input $X$ is not the raw input sequence, but rather an intermediate representation, such as the output from an input embedding layer or a preceding recurrent layer. The layer computes the output at each time step $t$ as
\begin{align*}
y[t]=\mathcal{R}(X)[t]=\sigma(A \mathcal{R}(X)[t-1] + B x[t] + c), \quad t = 1, \ldots, N,
\end{align*}
where $\sigma(x) = \max\{x, 0\}$ is the element-wise ReLU activation function, $A, B \in \mathbb{R}^{W \times W}$ are the weight matrices, and $c \in \mathbb{R}^{W}$ is the bias vector. The initial state $\mathcal{R}(X)[0]$ can be an arbitrary constant vector, which, for simplicity, is the zero vector $0 \in \mathbb{R}^{W}$ in our setting.

A deep recurrent neural network $\mathcal{N}$ is a composition of an input embedding map $\mathcal{P}$, an output embedding map $\mathcal{Q}$, and $L$ recurrent layers $\mathcal{R}_1, \ldots, \mathcal{R}_L$ (see Figure \ref{figure: 1} for an illustration), defined as
\begin{align}
\label{eq: 1}
\mathcal{N} := \mathcal{Q} \circ \mathcal{R}_L \circ \cdots \circ \mathcal{R}_1 \circ \mathcal{P}.
\end{align}
We denote $\mathcal{RNN}$ as a class of recurrent neural networks:
\begin{align*}
\mathcal{RNN}_{d_x, d_y}(W, L) = \{\mathcal{N}: \mathcal{N}(X) \text{ is of the form (\ref{eq: 1})}\}.
\end{align*}
The numbers $W$ and $L$ are called the width and depth of the network, respectively. We often omit the subscripts and simply denote it by $\mathcal{RNN}(W, L)$, when the input dimension $d_x$ and output dimension $d_y$ are clear from contexts.

In the following proposition, we summarize some basic operations on recurrent neural networks. These operations will be useful for construction, when we study the approximation capacity of RNNs. The proof is provided in Appendix \ref{sec: A.1}.

\begin{proposition}
\label{proposition: 1}
Let $\mathcal{N}_1 \in \mathcal{RNN}_{d_{x,1}, d_{y,1}}(W_1, L_1)$ and $\mathcal{N}_2 \in \mathcal{RNN}_{d_{x,2}, d_{y,2}}(W_2, L_2)$.
\begin{enumerate}[label=(\roman*)]
\item Inclusion: If $d_{x,1}=d_{x,2}, d_{y,1}=d_{y,2}, W_1 \leq W_2$ and $L_1 \leq L_2$, then $\mathcal{RNN}_{d_{x,1}, d_{y,1}}(W_1, L_1) \subseteq \mathcal{RNN}_{d_{x,2}, d_{y,2}}(W_2, L_2)$. 

\item Composition: If $d_{y,1}=d_{x,2}$, then $\mathcal{N}_2 \circ \mathcal{N}_1 \in \mathcal{RNN}_{d_{x,1}, d_{y,2}}(\max\{W_1, W_2\}, L_1+L_2)$. 

\item Concatenation: If $d_{x,1}=d_{x,2}$, define $\mathcal{N}(X):=(\mathcal{N}_1(X), \mathcal{N}_2(X))$, then $\mathcal{N} \in \mathcal{RNN}_{d_{x,1}, d_{y,1}+d_{y,2}}(W_1+W_2, \max\{L_1, L_2\})$.

\item Linear combination: If $d_{x,1}=d_{x,2}$ and $d_{y,1}=d_{y,2}$, then for any $c_1, c_2 \in \mathbb{R}$, $c_1 \mathcal{N}_1 + c_2 \mathcal{N}_2 \in \mathcal{RNN}_{d_{x,1},d_{y,1}} (W_1+W_2, \max\{L_1, L_2\})$.

\end{enumerate}
\end{proposition}

\subsection{Provable equivalence of deep RNNs and FNNs}

This subsection introduces the inclusion relationship between the function classes of RNNs and FNNs. A feedforward neural network $\mathcal{N}: \mathbb{R}^{d_x} \rightarrow \mathbb{R}^{d_y}$ can be expressed as
\begin{align}
\label{eq: 5}
\mathcal{N}(x) := \mathcal{F}_L \circ \cdots \circ \mathcal{F}_1 (x), \quad x \in \mathbb{R}^{d_x},
\end{align}
where each feedforward layer $\mathcal{F}_l (x)$ is defined as
\begin{align*}
\mathcal{F}_l (x) = \begin{cases}
\sigma(A_l x + b_l) & l = 1,\ldots,L-1  \\
A_l x + b_l & l = L
\end{cases},
\end{align*}
$A_l \in \mathbb{R}^{d_l \times d_{l-1}} (d_0 = d_x, d_L = d_y)$ is the weight matrix, and $b_l \in \mathbb{R}^{d_l}$ is the bias vector. The width $W$ of the network is defined as the maximum width of the hidden layers, that is, $W = \max\{d_1, \ldots, d_{L-1}\}$. Similarly, we denote $\mathcal{FNN}$ as a class of feedforward neural networks:
\begin{align*}
\mathcal{FNN}_{d_x,d_y} (W, L) = \{\mathcal{N}: \mathcal{N}(x) \text{ is of the form (\ref{eq: 5})}\}.
\end{align*}
It is noted that, unlike an RNN, the input to an FNN is typically a vector rather than a sequence (i.e., a matrix). To ensure consistency in the input format, if the input to an FNN is a sequence $X \in \mathbb{R}^{d_x \times N}$, we first stack its columns into a vector and then use this vector as the input to the neural network, that is, $\mathcal{N}(X) = \mathcal{N} \circ \operatorname{vec}(X)$, where $\operatorname{vec}(X) = (x[1]; \ldots; x[N]) \in \mathbb{R}^{d_x N}$. In contexts where it is clear, we do not distinguish between using a vector or a sequence as the network's input.

The following proposition illustrates that any RNN can be represented by a slightly larger FNN, and conversely, any FNN can be represented by a slightly larger RNN. The proof is provided in Section \ref{sec: 3}.

\begin{proposition}
\label{proposition: 3}
Let $t_0 \in \{1,\ldots,N\}$ and $W, L \in \mathbb{N}$.
\begin{enumerate}[label=(\roman*)]
\item For any FNN $\widetilde{\mathcal{N}} \in \mathcal{FNN}_{d_x \times t_0, d_y}(W, L)$, there exists an RNN $\mathcal{N} \in \mathcal{RNN}_{d_x, d_y}((d_x+1)W+1, 2L+2N)$ such that 
\begin{align*}
\widetilde{\mathcal{N}}(x[1:t_0]) = \mathcal{N}(X)[t_0], \quad X \in [0,1]^{d_x \times N}.
\end{align*}

\item For any RNN $\mathcal{N} \in \mathcal{RNN}_{d_x, d_y}(W, L)$, there exists an FNN $\widetilde{\mathcal{N}} \in \mathcal{FNN}_{d_x \times t_0, d_y}((2t_0-1)W, t_0 L)$ such that 
\begin{align*}
\mathcal{N}(X)[t_0] = \widetilde{\mathcal{N}}(x[1:t_0]), \quad X \in \mathbb{R}^{d_x \times N}.
\end{align*}
\end{enumerate}
\end{proposition}

If we denote $\mathcal{V}_{t_0}(W, L) = \{\mathcal{N}(X)[t_0]: \mathcal{N} \in \mathcal{RNN}_{d_x, d_y}(W, L)\}$ and restrict $X$ to $[0,1]^{d_x \times N}$, then for sufficiently large integers $\widebar{W} \geq d_x + 2$ and $\widebar{L} \geq 2N + 2$, Proposition \ref{proposition: 3} states that
\begin{align*}
\mathcal{FNN}_{d_x \times t_0, d_y}(\lfloor \tfrac{\widebar{W} - 1}{d_x+1} \rfloor, \lfloor \tfrac{\widebar{L}}{2} - N \rfloor) \,\subseteq\, \mathcal{V}_{t_0}(\widebar{W}, \widebar{L}) \,\subseteq\, \mathcal{FNN}_{d_x \times t_0, d_y}((2t_0-1)\widebar{W}, t_0 \widebar{L}).
\end{align*}
Intuitively speaking, the RNN and FNN classes are of comparable size. To the best of our knowledge, Proposition \ref{proposition: 3} provides the first equivalence result between RNNs and FNNs. This intuition further guides our analysis of the approximation and generalization properties of RNNs in the subsequent subsections. On the one hand, if an FNN can effectively approximate a target function, then any RNN containing this FNN can also approximate it. On the other hand, the complexity of the RNN class is controlled by that of the larger FNN class that contains it. In many applications, RNNs are used to simulate discrete-time dynamical systems because they are specifically designed to capture temporal dependencies between tokens. Proposition \ref{proposition: 3} demonstrates that RNNs and FNNs can be converted into each other, which helps in understanding how RNNs model sequences.

\subsection{Approximation of sequence-to-sequence functions}

This subsection introduces approximating a sequence of past-dependent H\"older functions using RNNs. The main idea is to utilize existing results on function approximation by FNNs. H\"older classes are widely studied in nonparametric statistics and are also the focus of this paper.

\begin{definition}[H\"older classes]
Let $\Omega \subseteq \mathbb{R}^{d_x}$ and $\gamma>0$ with $\gamma=r+\omega$, where $r \in \mathbb{N}_0$ and $\omega \in (0,1]$. A function is said to be $\gamma$-smooth if all its partial derivatives up to order $r$ exist and are bounded, and the partial derivatives of order $r$ are $\omega$-H\"{o}lder continuous. For $d_x,d_y \in \mathbb{N}$, the H\"{o}lder class with smoothness index $\gamma$ is then defined as
\begin{align}
\label{eq: 15}
\begin{aligned}
& \mathcal{H}_{d_x,d_y}^\gamma (\Omega, K) = \Bigg\{ f = (f_1, \ldots, f_{d_y})^\top: \Omega \rightarrow \mathbb{R}^{d_y}, \\ 
& \sum_{n: \|n\|_1 < \gamma} \|\partial^{n} f_k\|_{L^\infty(\Omega)} + \sum_{n: \|n\|_1 = r} \sup_{x, y \in \Omega,
x \neq y} \frac{|\partial^{n} f_k(x)-\partial^{n} f_k(y)|}{\|x - y\|^{\omega}} \leq K, \quad k=1, \ldots, d_y \Bigg\},
\end{aligned}
\end{align}
where $\partial^{n} = \partial^{n_1} \ldots \partial^{n_{d_x}}$ with $n = (n_1, \ldots, n_{d_x})^{\top} \in \mathbb{N}_0^{d_x}$ and $\|n\|_1 = \sum_{i=1}^{d_x} n_i$. 
\end{definition}

We note that, given an RNN $\mathcal{N}$, each token $y[t]$ of the output sequence $Y = \mathcal{N}(X)$ depends only on the prefix $x[1:t] = (x[1], x[2], \ldots, x[t])$ of the input sequence $X = (x[1], \ldots, x[N])$, owing to the recurrent structure. We refer to this property as past dependency.

\begin{definition}[Past-dependency]
Let $d_x, d_y \in \mathbb{N}$ and $\Omega \subseteq \mathbb{R}^{d_x}$. Given a sequence $X = (x[1], \ldots, x[N])$ of length $N$, a sequence of functions $\{f^{(t)}\}_{t=1}^N$ is said to be past-dependent if each $f^{(t)}: \Omega^{t} \rightarrow \mathbb{R}^{d_y}$ depends only on the first $t$ tokens of $X$, i.e., $x[1:t] = (x[1], \ldots, x[t])$.
\end{definition}

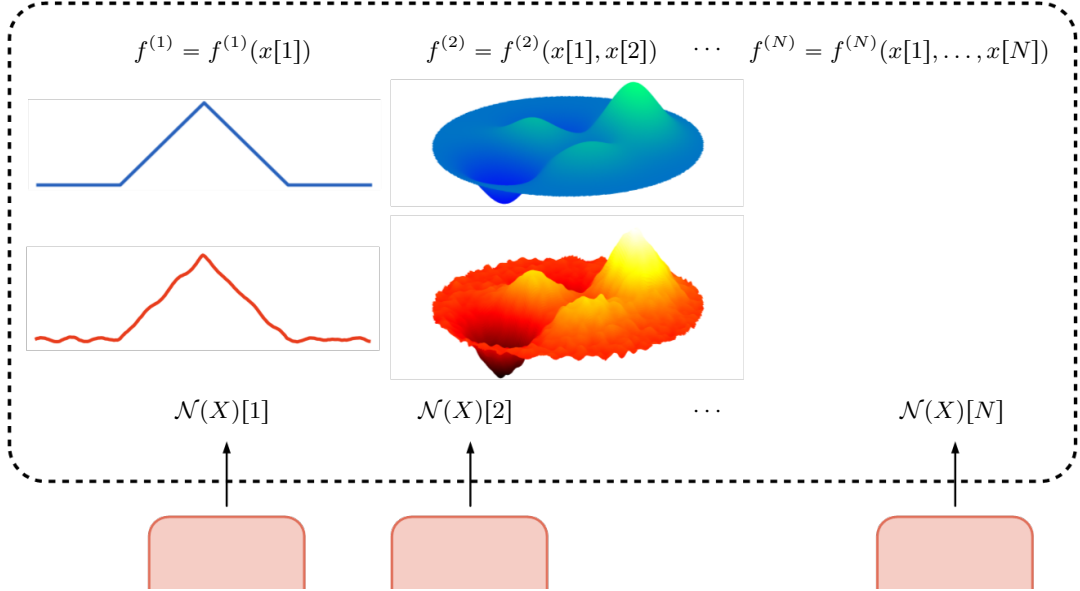
\begin{figure}[t]
\centering
\begin{tikzpicture}
    \node (image) at (0,0) {\includegraphics[width=0.9\linewidth]{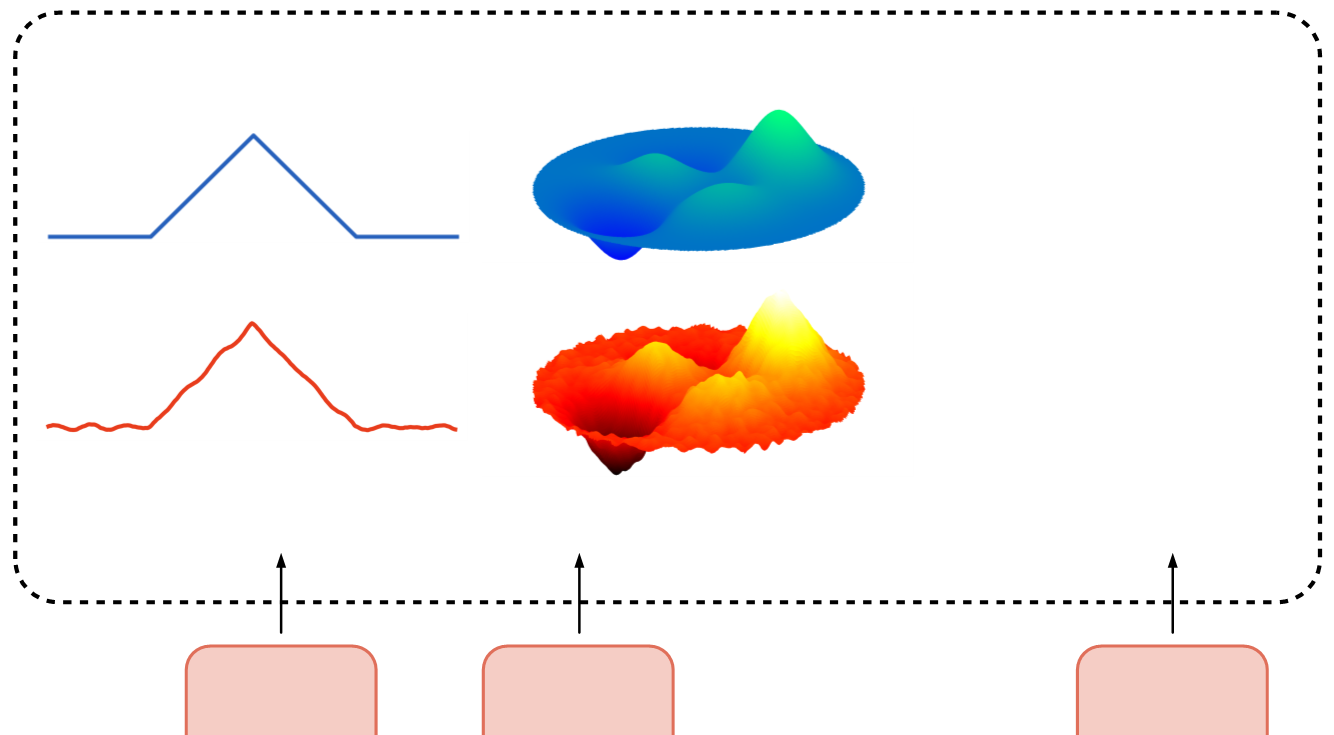}};
    \node at (-4.2, 3.2) {\footnotesize $f^{(1)} = f^{(1)}(x[1])$};
    \node at (0, 3.2) {\footnotesize $f^{(2)} = f^{(2)}(x[1], x[2])$};
    \node at (2.2, 3.2) {\footnotesize $\cdots$};
    \node at (4.7, 3.2) {\footnotesize $f^{(N)} = f^{(N)}(x[1], \ldots, x[N])$};

    \node at (-4.2, -1.6) {\footnotesize $\mathcal{N}(X)[1]$};
    \node at (-1, -1.6) {\footnotesize $\mathcal{N}(X)[2]$};
    \node at (2.2, -1.6) {\footnotesize $\cdots$};
    \node at (5.4, -1.6) {\footnotesize $\mathcal{N}(X)[N]$};
\end{tikzpicture}
\caption{An illustration of Theorem \ref{theorem: 1}. It holds simultaneously that $\mathcal{N}(X)[1] \approx f^{(1)}(x[1])$, $\mathcal{N}(X)[2] \approx f^{(2)}(x[1], x[2])$, and similarly, $\mathcal{N}(X)[N] \approx f^{(N)}(x[1], \ldots, x[N])$.}
\label{figure: 2}
\end{figure}

The following theorem shows that RNNs can simultaneously approximate a sequence of past-dependent H\"older functions. The proof is provided in Section \ref{sec: 4}.

\begin{theorem}\label{theorem: 1}
Given a sequence of past-dependent H\"older functions $\{f^{(t)}\}_{t=1}^N$ with each $f^{(t)}: [0,1]^{d_x \times t} \rightarrow \mathbb{R}^{d_y} \in \mathcal{H}_{d_x \times t,d_y}^\gamma ([0,1]^{d_x \times t}, K)$, for any $I,J \in \mathbb{N}$ with $I \geq \lceil\frac{4\gamma}{d_x}\rceil$ and $J \geq \lceil\exp(\frac{2\gamma}{d_x})\rceil$, there exists a recurrent neural network $\mathcal{N} \in \mathcal{RNN}_{d_x, d_y}(W,L)$ with width 
\begin{align*}
W = 153 (\lfloor\gamma\rfloor+1)^2 3^{d_x N} d_x^{\lfloor\gamma\rfloor+2} d_y N^{\lfloor\gamma\rfloor+2} J \lceil\log_2(8 J)\rceil
\end{align*}
and depth 
\begin{align*}
L = 43(\lfloor\gamma\rfloor + 1)^2 I \lceil\log_2(8 I)\rceil + 6 d_x N
\end{align*}
such that for each $t \in \{1,\ldots,N\}$,
\begin{align*}
\sup_{X \in [0,1]^{d_x \times N}} \|\mathcal{N}(X)[t] - f^{(t)}(x[1:t])\|_{\infty} \leq \left(24 N K^2 + 19 K(\lfloor\gamma\rfloor+1)^2 (d_x N)^{\lfloor\gamma\rfloor+(\gamma \vee 1)/2}\right)  (J I)^{-2 \gamma / (d_x t)}.
\end{align*}
\end{theorem}

Omitting constant factors in $d_x$, $d_y$, $\gamma$, $K$ and $N$, and focusing on the tunable parameters $I$ and $J$, Theorem \ref{theorem: 1} states that, given a sequence of past-dependent H\"older functions $\{f^{(t)}\}_{t=1}^N$, there exists a recurrent neural network $\mathcal{N} \in \mathcal{RNN}_{d_x, d_y}(W,L)$ with width $W \asymp J \log J$ and depth $L \asymp I \log I$ such that for each time step $t \in \{1, \ldots, N\}$,
\begin{align*}
\sup_{X \in [0,1]^{d_x \times N}} \|\mathcal{N}(X)[t] - f^{(t)}(x[1:t])\|_{\infty} \lesssim (J I)^{-2 \gamma / (d_x t)}.
\end{align*}
See Figure \ref{figure: 2} for an illustration. To the best of our knowledge, Theorem \ref{theorem: 1} is the first to propose the approximation rate of deep ReLU RNNs. There are many universal approximation results for RNNs, such as \cite{matthews1993approximating, Doya1993UniversalityOF, schafer2007recurrent, hoon2023minimal} in the discrete-time setting and \cite{funahashi1993approximation, chow2000modeling, li2005approximation, maass2007computational, nakamura2009approximation, li2022approximation} in the continuous-time setting. Our result provides a more detailed characterization of the approximation capacity of RNNs, addressing the question of how large a network needs to be in order to ensure that the approximation error is smaller than a specified precision.

A key concern is whether the exponential dependence on $d_x$ and $t$ (or $N$) in both the approximation rate and the prefactor is necessary. First, we claim that the prefactor $3^{d_x N}$ in the width $W$ can be removed at the expense of weakening the approximation from the $L^\infty$-norm to the $L^p$-norm for $p \in [1,\infty)$, following the techniques in \cite[Theorem 3.3]{jiao2023deep} and \cite[Theorem 1.1]{shen2022optimal}. Second, we claim that the approximation rate $(JI)^{-2 \gamma /(d_x t)}$ cannot be improved in general, implying that the curse of dimensionality is unavoidable without further assumptions; see Section \ref{sec: 2.4} for a detailed discussion.

We have established nearly optimal approximation rates for H\"older continuous functions approximated by deep ReLU RNNs. There are two possible directions to improve the approximation rates or reduce the effect of the curse of dimensionality. First, one may impose additional intrinsic structures on the target functions, such as the Barron integral representation \cite{barron1993universal, weinan2019priori, wojtowytsch2020representation}, the variation integral representation \cite{yang2024nonparametric, yang2024optimal, siegel2025optimal}, low-dimensional manifold assumptions \cite{chen2022nonparametric, shen2020deep, jiao2023deep}, low Minkowski dimension assumptions \cite{huang2022error, jiao2023deep}, or hierarchical structures \cite{bauer2019deep, schmidt2020nonparametric, kohler2023rate}. Under these assumptions, the approximation rates typically depend on the intrinsic rather than the ambient dimension. Second, one may consider neural networks with specially designed architectures. For instance, \cite{yarotsky2020phase, shen2021deep, shen2021neural} constructed FNNs with alternative activation functions including the floor function ($\lfloor x \rfloor$), sine function ($\sin x$), step function ($1_{x \geq 0}$) and exponential function ($2^x$), and showed that these FNNs can achieve exponentially small approximation errors for Lipschitz functions. Despite their remarkable approximation power, these networks typically have infinite VC-dimension and pose significant challenges for training in practice.

We also remark that, since our results rely on explicit constructions, we require $I \geq \lceil 4\gamma / d_x \rceil$ and $J \geq \lceil \exp(2\gamma / d_x) \rceil$, so that the minimal achievable width $W$ and depth $L$ are bounded away from some constants. Recent studies have shown that shallow ReLU FNNs have optimal approximation abilities \cite{mao2023rates, yang2024nonparametric, yang2024optimal}. Whether shallow RNNs also have strong approximation abilities remains an open and intriguing question, which is not covered by our current results.

\subsection{Optimality of Theorem \ref{theorem: 1}}\label{sec: 2.4}

In this subsection, we argue that the approximation rate established in Theorem \ref{theorem: 1} cannot, in general, be improved. Before proceeding to the proofs, we introduce the concept of covering number, which quantifies the complexity of the RNN class.

\begin{definition}[Covering number]
Let $\rho$ be a pseudo-metric on $\mathcal{M}$ and $S \subseteq \mathcal{M}$. For any $\delta>0$, a set $A \subseteq \mathcal{M}$ is called a $\delta$-covering of $S$ if for any $x \in S$ there exists $y \in A$ such that $\rho(x, y) \leq \delta$. The $\delta$-covering number of $S$, denoted by $\mathcal{N}(\delta, S, \rho)$, is the minimum cardinality of any $\delta$-covering of $S$.
\end{definition}

The following lemma provides an estimate of the covering number of the RNN class. Since covering numbers are defined for classes of real-valued functions, we restrict our attention to the case $d_y = 1$ and consider the output at a fixed time step $t_0$. The proof builds on Proposition \ref{proposition: 3}, which shows that any RNN can be represented as a particular FNN. Consequently, the RNN class can be embedded into a larger FNN class, for which covering number bounds are well established and sharp upper estimates are available (see, e.g., \cite{bartlett2019nearly, anthony2009neural}). This observation allows us to derive the covering number bounds for the RNN class. The proof is provided in Appendix \ref{sec: A.6}.

\begin{lemma}
\label{lemma: 8}
Let $K \geq 1$, $\delta \in (0,1)$, and $t_0 \in \{1, \ldots, N\}$. Define $\mathcal{V}_{t_0} = \{\mathcal{N}(X)[t_0]: \mathcal{N} \in \mathcal{RNN}_{d_x,1}(W,L)\}$. 

\begin{enumerate}[label=(\roman*)]
\item For fixed sequences $\mathcal{Z}=\{Z_i\}_{i=1}^n$ with each $Z_i \in [0,1]^{d_x \times N}$, define the pseudo-metric
\begin{align}
\label{eq: 12}
d_{\mathcal{Z}, \infty}(f,g) = \max_{1 \leq i \leq n} |f(Z_i)-g(Z_i)|.
\end{align}
Then, 
\begin{align*}
\sup_{\mathcal{Z}: |\mathcal{Z}|=n} \log \mathcal{N} (\delta, \mathcal{T}_K \mathcal{V}_{t_0}, d_{\mathcal{Z}, \infty}) \leq c_1 t_0^4 W^2 L^2 \log (t_0 W L) \log (K n / \delta),
\end{align*}
where $c_1>0$ is a universal constant.

\item Let $p \in [1,\infty)$, and $\lambda$ denote the uniform distribution on $[0,1]^{d_x \times N}$. Define the pseudo-metric
\begin{align}\label{eq: 25}
d_p(f, g) = \|f - g\|_{L^p(\lambda)}.
\end{align}
Then,
\begin{align*}
\log \mathcal{N} (\delta, \mathcal{T}_K \mathcal{V}_{t_0}, d_p) \leq c_2 p t_0^4 W^2 L^2 \log (t_0 W L) \log (K / \delta),
\end{align*}
where $c_2>0$ is a universal constant.
\end{enumerate}
\end{lemma}



Lemma \ref{lemma: 8} provides a new approach to estimate the covering number of the RNN class under different metrics (and, potentially, other sample complexity measures such as the Rademacher complexity) via function class inclusion. Compared with the existing bound in \cite[Lemma 3]{chen2020generalization}, our result differs in several aspects. First, we do not require an explicit norm constraint on the parameter matrices, and thus do not restrict the expressive power of RNNs. Second, our upper bound grows polynomially in $t_0$, $W$ and $L$, whereas the bound in \cite[Lemma 3]{chen2020generalization} exhibits exponential growth in $t_0$. Third, our estimate applies to both shallow ($L=1$) and deep ($L \geq 2$) RNNs, while the calculation in \cite[Lemma 3]{chen2020generalization} becomes considerably more involved for deeper architectures. It is worth noting, however, that the dependence on $t_0$ in our upper bound is of order $t_0^4 \leq N^4$. Hence, for very long sequences (i.e., large $N$), the estimate may be unsatisfactory. It remains unclear whether this dependence on $N$ is tight.

The following lemma shows that if the approximation error of a function class $\mathcal{F}$ with respect to the H\"older class $\mathcal{H}_{d,1}^\gamma$ is sufficiently small, then $\mathcal{F}$ must have a large covering number. In other words, a function class that is capable of uniformly approximating a rich class such as $\mathcal{H}_{d,1}^\gamma$ cannot be simple, and its metric entropy is necessarily lower bounded by that of the approximated class. This idea was used in \cite[Theorem 2.4]{lu2021deep} and \cite{yarotsky2017error} to establish the tightness of approximation bounds for smooth functions. The proof is provided in Appendix \ref{sec: A.7}.

\begin{lemma}\label{lemma: 10}
Let $\delta \in (0,1)$, $p \in [1,\infty)$, $d \in \mathbb{N}$, and $\gamma > 0$. Let $\mathcal{F}$ be a class of real-valued functions defined on $[0,1]^d$, and let $d_p$ denote the $L^p(\lambda)$ pseudo-metric defined in (\ref{eq: 25}). If for any $h \in \mathcal{H}_{d,1}^\gamma ([0,1]^{d}, 1)$, there exists $f \in \mathcal{F}$ such that $\|f - h\|_{L^p(\lambda)} \leq \delta$, then
\begin{align*}
\log \mathcal{N}(\delta, \mathcal{F}, d_p) \geq c \delta^{-d/\gamma},
\end{align*}
where $c$ is a constant depending only on $d$ and $\gamma$.
\end{lemma}

Combining Lemmas \ref{lemma: 8} and \ref{lemma: 10}, we obtain a lower bound for the approximation error. Let $p \in [1,\infty)$ and $t \in \{1,\ldots,N\}$. Choosing $d = d_x \times t$ and $\delta = (W L \log(W L))^{-2 \gamma / (d_x t)}$, by Lemma \ref{lemma: 8} we have
\begin{align*}
\log \mathcal{N} (\delta, \mathcal{T}_1 \mathcal{V}_{t}, d_p) \lesssim W^2 L^2 \log (W L) \log (1 / \delta) \lesssim (W L \log (W L))^2 = \delta^{-(d_x t) / \gamma}.
\end{align*}
Applying the contrapositive of Lemma \ref{lemma: 10}, it follows that there exists a function $f \in \mathcal{H}_{d_x \times t,1}^\gamma([0,1]^{d_x \times t}, 1)$ such that for all $\mathcal{N} \in \mathcal{RNN}_{d_x,1}(W,L)$,
\begin{align*}
\|\mathcal{N}(X)[t] - f\|_{L^p(\lambda)} \geq \|\mathcal{T}_1 \mathcal{N}(X)[t] - f\|_{L^p(\lambda)} \gtrsim \delta = (W L \log(W L))^{-2 \gamma / (d_x t)}.
\end{align*}
Hence, the best possible approximation error of functions in $\mathcal{H}_{d_x \times t,1}^\gamma ([0,1]^{d_x \times t}, 1)$, approximated by ReLU RNNs with width $W$ and depth $L$ at time step $t$, is of order
\begin{align*}
(W L \log(W L))^{-2 \gamma / (d_x t)}.
\end{align*}
If we set width $W \asymp J \log J$ and depth $L \asymp I \log I$, then the lower bound of the approximation error becomes
\begin{align*}
\left(J I (\log J)^2 (\log I)^2\right)^{-2 \gamma / (d_x t)}.
\end{align*}
Comparing this lower bound with the upper bound established in Theorem \ref{theorem: 1}, we conclude that our approximation results are nearly optimal in the $L^p$-norm at each time step $t$.

\subsection{Regression with dependent data}

We consider the following $N$-step prediction model
\begin{align}
\label{eq: 10}
Y = f^*(X_1, X_2, \ldots, X_N) + \varepsilon, 
\end{align}
where $Y \in \mathbb{R}$ is a response, $f^*(x_1, \ldots, x_N) = \mathbb{E}[Y | X_1=x_1, \ldots, X_N=x_N]: [0,1]^{d_x \times N} \rightarrow \mathbb{R}$ is an unknown regression function and $\varepsilon$ is a sub-Gaussian noise, independent of $X_i, i=1,\ldots,N$, with $\mathbb{E}[\varepsilon] = 0$ and 
\begin{align*}
\mathbb{E} [\exp(s\varepsilon)] \leq \exp\left(\frac{\sigma^2 s^2}{2}\right), \quad s \in \mathbb{R}.
\end{align*}
Our purpose is to estimate the unknown target regression function $f^*$ given observations $\mathcal{D}_n = \{(x_1, y_1), \ldots, (x_n, y_n)\}$ which may not be i.i.d.

The dependence arises from two aspects. First, we assume $\mathcal{X} = \{x_t\}_{t=1}^n$ is a stationary $\beta$-mixing sequence, which is a relaxation of the classical i.i.d. setting. The following definitions are required.

\begin{definition}[Stationarity]
\label{definition: 1}
A sequence of random variables $\{x_t\}_{t=-\infty}^{\infty}$ is said to be stationary if for any $t$ and non-negative integers $m$ and $k$, the random vectors $(x_t, \ldots, x_{t+m})$ and $(x_{t+k}, \ldots, x_{t+m+k})$ have the same distribution.
\end{definition}

\begin{definition}[$\beta$-mixing]
\label{definition: 2}
Let $\{x_t\}_{t=-\infty}^{\infty}$ be a stationary sequence of random variables. For any $i, j \in \mathbb{Z} \cup \{-\infty,+\infty\}$, let $\sigma_i^j$ denote the $\sigma$-algebra generated by the random variables $x_k, i \leq k \leq j$. Then, for any positive integer $k$, the $\beta$-mixing coefficient of the stochastic process $\{x_t\}_{t=-\infty}^{\infty}$ is defined as
\begin{align*}
\beta(k) = \sup_n \mathbb{E}_{B \in \sigma_{-\infty}^n} \left[\sup_{A \in \sigma_{n+k}^{\infty}} |\mathbb{P}(A \mid B)-\mathbb{P}(A)|\right].
\end{align*}
$\{x_t\}_{t=-\infty}^{\infty}$ is said to be $\beta$-mixing if $\beta(k) \rightarrow 0$ as $k \rightarrow \infty$. It is said to be algebraically $\beta$-mixing if there exist real numbers $\beta_0>0$ and $r>0$ such that $\beta(k) \leq \beta_0 / k^r$ for all $k$, and exponentially $\beta$-mixing if there exist real numbers $\beta_0, \beta_1>0$ and $r>0$ such that $\beta(k) \leq \beta_0 \exp \left(-\beta_1 k^r\right)$ for all $k$.
\end{definition}

Thus, by Definition \ref{definition: 1}, the time index $t$ does not affect the distribution of $x_t$ in a stationary sequence. Additionally, any $N$ consecutive variables $(x_{t-N+1}, \ldots, x_t)$ have the same joint distribution. Definition \ref{definition: 2} provides a standard measure of the dependence among the random variables $x_t$ within a stationary sequence. The scenario where observations are drawn from a stationary mixing distribution has become standard and has been discussed by many previous studies \cite{yu1994rates, meir2000nonparametric, mohri2008rademacher, steinwart2009fast, mohri2010stability, agarwal2012generalization, shalizi2013predictive, kuznetsov2017generalization, ren2024statistical}.

Second, we consider the sliding window training approach. A fundamental method for estimating $f^*$ is to minimize the mean squared error, or the $L^2$ risk, defined as
\begin{align*}
\mathcal{R}(f) = \mathbb{E}_{((X_1, \ldots, X_N) \sim \Pi, Y)} [(f(X_1, \ldots, X_N) - Y)^2],
\end{align*}
where $\Pi$ denotes the joint distribution of the predictors $(X_1, \ldots, X_N)$. To ease our analysis, we assume that $\Pi$ is supported on $[0,1]^{d_x \times N}$ and is absolutely continuous with respect to the Lebesgue measure. It can be shown that the underlying regression function $f^*$ is the global minimizer of the $L^2$ risk, that is,
\begin{align*}
f^* = \argmin{f}\, \mathcal{R}(f).
\end{align*}
However, in practical applications, the distribution of $((X_1, \ldots, X_N), Y)$ is typically unknown, and only a random sample $\mathcal{D}_n = \{(x_i, y_i)\}_{i=1}^n$ is available. Since each prediction of an RNN requires an input sequence of length $N$, we construct overlapping subsequences of length $N$ as follows:
\begin{align*}
\{((x_1, \ldots, x_N), y_N), ((x_2, \ldots, x_{N+1}), y_{N+1}), \ldots, ((x_{n-N+1}, \ldots, x_n), y_n)\}.
\end{align*}
Here, each subsequence $(x_{t-N+1}, \ldots, x_t)$ is assumed to follow the distribution $\Pi$ by the assumption of stationarity, and each pair $((x_{t-N+1}, \ldots, x_t), y_t)$ is assumed to follow the model in (\ref{eq: 10}). Based on these subsequences, we estimate the unknown target function $f^*$ by the constrained empirical risk minimizer (ERM),
\begin{align}
\label{eq: 20}
\hat{f} \in \argmin{f \in \mathcal{V}}\, \mathcal{R}_n(\mathcal{T}_K f) := \frac{1}{n-N+1} \sum_{t=N}^n (\mathcal{T}_K f(x_{t-N+1}, \ldots, x_t) - y_t)^2,
\end{align}
where the truncation operator $\mathcal{T}_K$ at level $K$ acts on any real-valued function $f$ as
\begin{align*}
\mathcal{T}_K f(X) := \begin{cases}
f(X) & \text{if } |f(X)| \leq K,\\
\operatorname{sgn}(f(X)) \cdot K & \text{if } |f(X)| > K.
\end{cases}
\end{align*}
For a class of real-valued functions $\mathcal{F}$, we write $\mathcal{T}_K \mathcal{F} := \{\mathcal{T}_K f: f \in \mathcal{F}\}$. We focus on the case where $\mathcal{V} = \{\mathcal{N}(X)[N]: \mathcal{N} \in \mathcal{RNN}_{d_x,1}(W,L)\}$, namely the set of outputs of RNNs at the last time step.

The quality of the estimator is evaluated by the excess risk, defined as the difference in the $L^2$ risk between $\mathcal{T}_K \hat{f}$ and $f^*$,
\begin{align*}
\mathcal{R}(\mathcal{T}_K \hat{f}) - \mathcal{R}(f^*) = \|\mathcal{T}_K \hat{f} - f^*\|_{L^2(\Pi)}^2.
\end{align*}
We can now present our main result, which characterizes the convergence rates for the estimation of $f^*$ using recurrent neural networks and dependent data. The proof is provided in Section \ref{sec: 5}.

\begin{theorem}
\label{theorem: 4}
Assume that $f^* \in \mathcal{H}_{d_x \times N, 1}^\gamma ([0,1]^{d_x \times N}, K)$. We choose the hypothesis class $\mathcal{V} = \{\mathcal{N}(X)[N]: \mathcal{N} \in \mathcal{RNN}_{d_x,1}(W,L)\}$ in (\ref{eq: 20}).

\begin{enumerate}[label=(\roman*)]
\item Assume that $\{x_t\}_{t=1}^n$ is an exponentially $\beta$-mixing sequence, i.e., $\beta(k) \leq \beta_0 \exp (-\beta_1 k^r)$ for all $k$. We set the parameters as
\begin{align*}
W \asymp n^\alpha \log n, \quad L \asymp n^{\frac{d_x N}{2 d_x N + 4 \gamma} - \alpha} \log n
\end{align*}
for fixed $0 \leq \alpha \leq \frac{d_x N}{2 d_x N + 4 \gamma}$. Then the ERM $\hat{f}$ satisfies
\begin{align*}
\mathbb{E}_{\mathcal{D}_n} [\|\mathcal{T}_K \hat{f} - f^*\|_{L^2(\Pi)}^2] & \lesssim n^{-\frac{2 \gamma}{d_x N + 2 \gamma}} (\log n)^{6+\frac{1}{r}}.
\end{align*}

\item Assume that $\{x_t\}_{t=1}^n$ is an algebraically $\beta$-mixing sequence, i.e., $\beta(k) \leq \beta_0 / k^r$ for all $k$. We set the parameters as
\begin{align*}
W \asymp n^\alpha \log n, \quad L \asymp n^{\frac{r d_x N}{(2r+2) d_x N + (4r+8) \gamma} - \alpha} \log n
\end{align*}
for fixed $0 \leq \alpha \leq \frac{r d_x N}{(2r+2) d_x N + (4r+8) \gamma}$. Then the ERM $\hat{f}$ satisfies
\begin{align*}
\mathbb{E}_{\mathcal{D}_n} [\|\mathcal{T}_K \hat{f} - f^*\|_{L^2(\Pi)}^2] & \lesssim n^{-\frac{2 r \gamma}{(r+1) d_x N + (2r+4) \gamma}} (\log n)^{6}.
\end{align*}

\item Assume that $\{x_t\}_{t=1}^n$ is a sequence of i.i.d. random variables. We set the parameters as
\begin{align*}
W \asymp n^\alpha \log n, \quad L \asymp n^{\frac{d_x N}{2 d_x N + 4 \gamma} - \alpha} \log n
\end{align*}
for fixed $0 \leq \alpha \leq \frac{d_x N}{2 d_x N + 4 \gamma}$. Then the ERM $\hat{f}$ satisfies
\begin{align*}
\mathbb{E}_{\mathcal{D}_n} [\|\mathcal{T}_K \hat{f} - f^*\|_{L^2(\Pi)}^2] \lesssim n^{-\frac{2 \gamma}{d_x N + 2 \gamma}} (\log n)^{6}.
\end{align*}
\end{enumerate}
\end{theorem}

Theorem \ref{theorem: 4} shows that the ERM $\hat{f}$ achieves the optimal minimax rate $n^{-\frac{2 \gamma}{d_x N + 2 \gamma}}$, up to logarithmic factors, in both exponentially $\beta$-mixing and i.i.d. settings. It has been established that the optimal minimax rates in the $L^2$-norm for i.i.d. data remain optimal for dependent sequences that satisfy certain $\beta$-mixing conditions \cite{yu1993density, viennet1997inequalities}. Our findings differ from prior studies in two important aspects. First, while similar rates of convergence under $\beta$-mixing conditions have been studied in \cite{feng2023over, ren2024statistical}, our focus is on RNNs rather than FNNs, and the proof techniques differ largely. In particular, \cite{feng2023over} obtained a suboptimal rate $n^{-\frac{\gamma}{2 \gamma + 2 d + 2}}$, where $d$ denotes the input dimension (so $d = d_x N$ in our setting), whereas our refined error decomposition yields the sharp rate $n^{-\frac{2 \gamma}{d_x N + 2 \gamma}}$. While \cite{ren2024statistical} also established this rate using localization techniques, their bounds in expectation only hold with high probability. By contrast, our analysis based on offset Rademacher complexity provides bounds in expectation without such restrictions. Moreover, \cite{feng2023over, ren2024statistical} did not consider the case where the $\beta$-mixing coefficients decay polynomially. Second, \cite{kohler2023rate} also studied rates of convergence using RNN models and showed that RNN-based estimators can overcome the curse of dimensionality. However, their results crucially rely on the assumption that the target function has a hierarchical structure. The RNN model defined therein, unlike ours, is tailored to this assumption. The role of hierarchical structure in circumventing the curse of dimensionality has also been discussed in \cite{bauer2019deep, schmidt2020nonparametric}. In general, the rate $n^{-\frac{2 \gamma}{d_x N + 2 \gamma}}$ cannot be improved without extra assumptions \cite{stone1982optimal, yu1993density, viennet1997inequalities}. Additionally, our findings offer greater flexibility in adjusting the width and depth of RNNs through the parameter $\alpha$. Our results can be easily extended to the fixed-width or fixed-depth cases, following \cite{jiao2023deep}.


In the algebraically $\beta$-mixing setting, our rate $n^{-\frac{2 r \gamma}{(r+1) d_x N + (2r+4) \gamma}}$ is suboptimal. Note that $\frac{2 r \gamma}{(r+1) d_x N + (2r+4) \gamma} \rightarrow \frac{2 \gamma}{d_x N + 2 \gamma}$ as $r \rightarrow \infty$. The rate of convergence in this case is ``close'' to the optimal rate $n^{-\frac{2 \gamma}{d_x N + 2 \gamma}}$, provided $r$ is sufficiently large.

\section{Proof of Proposition \ref{proposition: 3}}\label{sec: 3}

This section contains the proof of Proposition \ref{proposition: 3}. The proof techniques employed in this section are adapted from \cite{hoon2023minimal}. To begin with, we define a new activation function that acts on only certain components, rather than all components. With the ReLU activation function $\sigma$ and an index set $I \subseteq \mathbb{N}$, the modified activation $\sigma_I$ is defined as
\begin{align*}
\sigma_I(s)_i = \begin{cases} 
\sigma(s_i) & \text { if } i \in I, \\ 
s_i & \text { otherwise}.
\end{cases}
\end{align*}
Using the modified activation function $\sigma_I$, a modified recurrent layer is defined as
\begin{align*}
\begin{aligned}
\mathcal{R}(X)[t]_i & =\sigma_I(A \mathcal{R}(X)[t-1]+B x[t]+c)_i \\
& =\left\{\begin{array}{ll}
\sigma(A \mathcal{R}(X)[t-1]+B x[t]+c)_i & \text { if } i \in I, \\
(A \mathcal{R}(X)[t-1]+B x[t]+c)_i & \text { otherwise}.
\end{array}\right.
\end{aligned}
\end{align*}
We denote by $\mathcal{MRNN}_{d_x, d_y}(W,L)$ the class of neural network functions composed of modified recurrent layers, rather than original recurrent layers. It is clear that $\mathcal{RNN}_{d_x, d_y}(W,L) \subseteq \mathcal{MRNN}_{d_x, d_y}(W,L)$. The following lemma reveals the reverse inclusion relationship when the input $X$ is restricted to a compact set. 

\begin{lemma}
\label{lemma: 5}
For any MRNN $\widebar{\mathcal{N}} \in \mathcal{MRNN}_{d_x, d_y}(W, L)$, there exists an RNN $\mathcal{N} \in \mathcal{RNN}_{d_x, d_y}(W+1, 2L)$ such that
\begin{align*}
\widebar{\mathcal{N}}(X)[t] = \mathcal{N}(X)[t], \quad X \in [0,1]^{d_x \times N}, t = 1, \ldots, N.
\end{align*}
\end{lemma}

We next show that a feedforward neural network can be represented as a modified recurrent neural network. The proof is inspired by \cite[Lemma 6]{hoon2023minimal}.

\begin{lemma}
\label{lemma: 4}
Let $t_0\in\{1,\ldots,N\}$. For any FNN $\widetilde{\mathcal{N}}(x[1:t_0]) \in \mathcal{FNN}_{d_x \times t_0, d_y}(W, L)$, there exists an MRNN $\widebar{\mathcal{N}} \in \mathcal{MRNN}_{d_x,d_y}((d_x+1)W, N+L)$ such that
\begin{align*}
\widetilde{\mathcal{N}}(x[1:t_0]) = \widebar{\mathcal{N}}(X)[t_0],
\end{align*}
where $X = (x[1],\ldots,x[N]) \in \mathbb{R}^{d_x \times N}$.
\end{lemma}

Combining Lemma \ref{lemma: 5} and \ref{lemma: 4}, we conclude that any FNN can be represented by an MRNN. This MRNN can, in turn, be represented by an RNN, thereby proving the first part of Proposition \ref{proposition: 1}. For the second part, that an RNN can be represented by an FNN, we present the following lemma, which is based on the idea that, using the identity $\sigma(x) - \sigma(-x) = x$, we are able to store and copy the computation results at each step.

\begin{lemma}
\label{lemma: 7}
Let $t_0\in\{1,\ldots,N\}$. For any RNN $\mathcal{N} \in \mathcal{RNN}_{d_x,d_y}(W,L)$, there exists an FNN $\widetilde{\mathcal{N}} \in \mathcal{FNN}_{d_x \times t_0, d_y}((2t_0-1) W, t_0 L)$ such that
\begin{align*}
\mathcal{N}(X)[t_0] = \widetilde{\mathcal{N}}(x[1:t_0]),
\end{align*}
where $X = (x[1],\ldots,x[N]) \in \mathbb{R}^{d_x \times N}$.
\end{lemma}

Proofs of these lemmas are provided in Appendix \ref{sec: A.2}, \ref{sec: A.3} and \ref{sec: A.4}. Proposition \ref{proposition: 3} can be directly derived from the above results.

\begin{proof}[Proof of Proposition \ref{proposition: 3}]
It follows directly from Lemma \ref{lemma: 5}, \ref{lemma: 4} and \ref{lemma: 7}.
\end{proof}

\section{Proof of Theorem \ref{theorem: 1}}\label{sec: 4}

This section contains the proof of Theorem \ref{theorem: 1}, which shows that an RNN can approximate a sequence of past-dependent H\"older functions simultaneously. The basic idea is that if the output of an RNN at a time step can approximate a function, then by constructing a sequence of RNNs, each approximating a past-dependent function, and combining them into a larger RNN, we can simultaneously approximate a sequence of past-dependent functions. The problem then reduces to how to enable the output of an RNN at a single time step to approximate a function. We consider using FNNs as intermediaries. Proposition \ref{proposition: 3} demonstrates that any FNN can be represented by the output of an RNN at a single time step. The approximation of H\"older functions by FNNs has been well studied, for example in \cite{yarotsky2017error, lu2021deep, shen2022optimal, jiao2023deep}. The following lemma, proved in Appendix \ref{sec: A.5}, formalizes this idea.

\begin{lemma}
\label{lemma: 6}
Let $t_0\in\{1,\ldots,N\}$. Assume $f = (f_1, \ldots, f_{d_y})^\top \in \mathcal{H}_{d_x \times t_0,d_y}^\gamma ([0,1]^{d_x \times t_0}, K)$. Then for any $I, J \in \mathbb{N}$, there exists a recurrent neural network $\mathcal{N} \in \mathcal{RNN}_{d_x, d_y}(W,L)$ with width 
\begin{align*}
W = 152 (\lfloor\gamma\rfloor+1)^2 3^{d_x t_0} d_x^{\lfloor\gamma\rfloor+2} d_y t_0^{\lfloor\gamma\rfloor+1} J \lceil\log_2(8 J)\rceil
\end{align*}
and depth 
\begin{align*}
L = 42(\lfloor\gamma\rfloor + 1)^2 I \lceil\log_2(8 I)\rceil + 4 d_x t_0 + 2 N
\end{align*}
such that
\begin{align*}
\sup_{X \in [0,1]^{d_x \times N}}  \|f(x[1:t_0])-\mathcal{N}(X)[t_0]\|_{\infty} \leq 19 K(\lfloor\gamma\rfloor+1)^2 (d_x t_0)^{\lfloor\gamma\rfloor+(\gamma \vee 1) / 2}(J I)^{-2 \gamma / (d_x t_0)}.
\end{align*}
\end{lemma}

Unlike the simultaneous approximation in Theorem \ref{theorem: 1}, Lemma \ref{lemma: 6} considers approximation only at the $t_0$ time step. Next, we combine $N$ RNNs into a larger RNN to approximate a sequence of functions.


\begin{proof}[Proof of Theorem \ref{theorem: 1}]
For $t_0\in\{1,\ldots,N\}$ and $f^{(t_0)} \in \mathcal{H}_{d_x \times t_0,d_y}^\gamma ([0,1]^{d_x \times t_0}, K)$, Lemma \ref{lemma: 6} gives that for any $I, J \in \mathbb{N}$, there exists a recurrent neural network $\mathcal{N}^{(t_0)} \in \mathcal{RNN}_{d_x, d_y}(W^{(t_0)},L^{(t_0)})$ with width
\begin{align*}
W^{(t_0)} & = 152 (\lfloor\gamma\rfloor+1)^2 3^{d_x t_0} d_x^{\lfloor\gamma\rfloor+2} d_y t_0^{\lfloor\gamma\rfloor+1} J \lceil\log_2(8 J)\rceil \\
& \leq 152 (\lfloor\gamma\rfloor+1)^2 3^{d_x N} d_x^{\lfloor\gamma\rfloor+2} d_y N^{\lfloor\gamma\rfloor+1} J \lceil\log_2(8 J)\rceil
\end{align*}
and depth 
\begin{align*}
L^{(t_0)} & = 42(\lfloor\gamma\rfloor + 1)^2 I \lceil\log_2(8 I)\rceil + 4 d_x t_0 + 2 N \\
& \leq 42(\lfloor\gamma\rfloor + 1)^2 I \lceil\log_2(8 I)\rceil + 6 d_x N
\end{align*}
such that
\begin{align}
\label{eq: 6}
\sup_{X \in [0,1]^{d_x \times N}}  \|f^{(t_0)}(x[1:t_0]) - \mathcal{N}^{(t_0)}(X)[t_0]\|_{\infty} \leq 19 K(\lfloor\gamma\rfloor+1)^2 (d_x t_0)^{\lfloor\gamma\rfloor+(\gamma \vee 1) / 2}(J I)^{-2 \gamma / (d_x t_0)}.
\end{align}

Since $\|f^{(t_0)}\|_{\infty} \leq K$, we can truncate the output of $\mathcal{N}^{(t_0)}$ harmlessly by applying $\mathcal{T}_K(x) = \min\{\max\{x,-K\},K\}$ element-wise. Since
\begin{align*}
\mathcal{T}_K(x) = \sigma(x) - \sigma(-x) - \sigma(x-K) + \sigma(-x-K),
\end{align*}
it is not hard to see that the truncated network $\mathcal{N}^{(t_0)}_{\text{trunc}} :=  \mathcal{T}_K \circ \mathcal{N}^{(t_0)} \in \mathcal{RNN}_{d_x, d_y}(W^{(t_0)},L^{(t_0)}+1)$ by Proposition \ref{proposition: 1} and 
\begin{align}
\label{eq: 7}
\|f^{(t_0)}(x[1:t_0]) - \mathcal{N}^{(t_0)}_{\text{trunc}}(X)[t_0]\|_{\infty} \leq \|f^{(t_0)}(x[1:t_0]) - \mathcal{N}^{(t_0)}(X)[t_0]\|_{\infty}.
\end{align}

We next define the time step indicator $\mathcal{I}^{(t_0)}$ as an additional component, which outputs 1 at the $t_0$-th time step and 0 at all other time steps. Let
\begin{align*}
\begin{gathered}
P = O_{3,d_x}, \quad A_1 = \begin{pmatrix}
1 & 0 & 0 \\
0 & 0 & 0 \\
0 & 0 & 0
\end{pmatrix}, \quad 
B_1 = O_{3,3}, \quad c_1 = \begin{pmatrix}
1 \\
0 \\
0
\end{pmatrix}, \\
A_2 = O_{3,3}, \quad B_2 = \begin{pmatrix}
1 & 0 & 0 \\
1 & 0 & 0 \\
1 & 0 & 0
\end{pmatrix}, \quad c_2 = \begin{pmatrix}
-t_0+1 \\
-t_0 \\
-t_0-1
\end{pmatrix}, \quad Q = \begin{pmatrix}
1 & -2 & 1
\end{pmatrix},
\end{gathered}
\end{align*}
then $\mathcal{I}^{(t_0)} := \mathcal{Q} \circ \mathcal{R}_2 \circ \mathcal{R}_1 \circ \mathcal{P} \in \mathcal{RNN}_{d_x,1}(3,2)$. Direct computation yields
\begin{align}
\label{eq: 8}
\mathcal{I}^{(t_0)}(X)[t] = \sigma(t-t_0+1) - 2\sigma(t-t_0) + \sigma(t-t_0-1) = \delta_{t,t_0},
\end{align}
where $\delta$ is the Kronecker delta function. By concatenation as in Proposition \ref{proposition: 1}, we have $
(\mathcal{N}^{(t_0)}_{\text{trunc}}(X), \mathcal{I}^{(t_0)}(X)) \in \mathcal{RNN}_{d_x, d_y+1}(W^{(t_0)}+3,L^{(t_0)}+1)$.

FNNs can effectively approximate multiplication. Token-wise FNNs, as a special form of RNNs, apply this approximation uniformly across each input token. Specifically, according to \cite[Lemma 4.2]{lu2021deep}, there exists an FNN $\mathcal{N}_{\times}$ with width $d_y (9J+1)$ and depth $I$ such that
\begin{align*}
\|\mathcal{N}_{\times}(u, v) - uv\|_{\infty} \leq 24 K^2 J^{-I}
\end{align*}
for any $u \in [-K,K]^{d_y}, v \in [0,1]$. By composition, $\mathcal{N}_{\times}$ acts equally on each time step of $(\mathcal{N}^{(t_0)}_{\text{trunc}}(X), \mathcal{I}^{(t_0)}(X))$, leading to
\begin{align}
\label{eq: 9}
\|\mathcal{N}_{\times}(\mathcal{N}^{(t_0)}_{\text{trunc}}(X)[t], \mathcal{I}^{(t_0)}(X)[t]) - \mathcal{N}^{(t_0)}_{\text{trunc}}(X)[t] \ \mathcal{I}^{(t_0)}(X)[t]\|_{\infty} \leq 24 K^2 J^{-I}, \quad t = 1,\ldots,N,
\end{align}
and $\mathcal{N}_{\times}(\mathcal{N}^{(t_0)}_{\text{trunc}}(X), \mathcal{I}^{(t_0)}(X)) \in \mathcal{RNN}_{d_x, d_y}(W^{(t_0)}+3,L^{(t_0)}+1+I)$.

In the last, we prove that $\mathcal{N} := \sum_{t = 1}^N \mathcal{N}_{\times} (\mathcal{N}^{(t)}_{\text{trunc}}, \mathcal{I}^{(t)})$ can simultaneously approximate a set of past-dependent H\"older functions $\{f^{(t)}\}_{t=1}^N$. Combining (\ref{eq: 6}), (\ref{eq: 7}), (\ref{eq: 8}) and (\ref{eq: 9}), we have
\begin{align*}
& \|\mathcal{N}(X)[t_0] - f^{(t_0)}\|_{\infty} \\
& = \|\sum_{t = 1}^N \mathcal{N}_{\times} (\mathcal{N}^{(t)}_{\text{trunc}}(X), \mathcal{I}^{(t)}(X))[t_0] - f^{(t_0)}\|_{\infty} \\
& = \|\sum_{t = 1}^N \mathcal{N}_{\times} (\mathcal{N}^{(t)}_{\text{trunc}}(X)[t_0], \mathcal{I}^{(t)}(X)[t_0]) - f^{(t_0)}\|_{\infty} \\
& \leq \|\sum_{t = 1}^N \mathcal{N}_{\times} (\mathcal{N}^{(t)}_{\text{trunc}}(X)[t_0], \mathcal{I}^{(t)}(X)[t_0]) - \sum_{t = 1}^N \mathcal{N}^{(t)}_{\text{trunc}}(X)[t_0] \  \mathcal{I}^{(t)}(X)[t_0]\|_{\infty} + \|\mathcal{N}^{(t_0)}_{\text{trunc}}(X)[t_0] - f^{(t_0)}\|_{\infty} \\
& \leq \sum_{t = 1}^N \|\mathcal{N}_{\times} (\mathcal{N}^{(t)}_{\text{trunc}}(X)[t_0], \mathcal{I}^{(t)}(X)[t_0]) - \mathcal{N}^{(t)}_{\text{trunc}}(X)[t_0] \  \mathcal{I}^{(t)}(X)[t_0]\|_{\infty} + \|\mathcal{N}^{(t_0)}(X)[t_0] - f^{(t_0)}\|_{\infty} \\
& \leq 24 N K^2 J^{-I} + 19 K(\lfloor\gamma\rfloor+1)^2 (d_x t_0)^{\lfloor\gamma\rfloor+(\gamma \vee 1)/2}(J I)^{-2 \gamma / (d_x t_0)} \\
& \leq \left(24 N K^2 + 19 K(\lfloor\gamma\rfloor+1)^2 (d_x N)^{\lfloor\gamma\rfloor+(\gamma \vee 1)/2}\right)  (J I)^{-2 \gamma / (d_x t_0)},
\end{align*}
where in the last inequality we use the fact that $I \geq \lceil\frac{4\gamma}{d_x}\rceil$ and $J \geq \lceil\exp\{\frac{2\gamma}{d_x}\}\rceil$ imply $J^{-I} \leq (J I)^{-2 \gamma / (d_x t_0)}$. According to our construction, $\mathcal{N} \in \mathcal{RNN}_{d_x, d_y}(W,L)$ with width 
\begin{align*}
W = \sum_{t=1}^N (W^{(t)}+3) \leq 153 (\lfloor\gamma\rfloor+1)^2 3^{d_x N} d_x^{\lfloor\gamma\rfloor+2} d_y N^{\lfloor\gamma\rfloor+2} J \lceil\log_2(8 J)\rceil
\end{align*}
and depth 
\begin{align*}
L = \max_{t} L^{(t)}+1+I \leq 43(\lfloor\gamma\rfloor + 1)^2 I \lceil\log_2(8 I)\rceil + 6 d_x N
\end{align*}
by Proposition \ref{proposition: 1}.
\end{proof}

\section{Proof of Theorem \ref{theorem: 4}}\label{sec: 5}

This section contains the proof of Theorem \ref{theorem: 4}. The following theorem provides a decomposition of the excess risk into three components: approximation error, generalization error, and dependence error. This result can be viewed as an extension of the classical decomposition in the i.i.d. case to the $\beta$-mixing assumption. Recall that $d_{\mathcal{Z}, \infty}$ is defined in (\ref{eq: 12}). The proof is provided in Appendix~\ref{sec: A.8}.

\begin{theorem}[Oracle inequality]
\label{theorem: 3}
Let $\mathcal{V}$ be a class of functions. Assume that $f^* \in \mathcal{H}_{d_x \times N, 1}^\gamma ([0,1]^{d_x \times N}, K)$, $\{x_t\}_{t=1}^n$ is a stationary $\beta$-mixing sequence, and $n \geq 4 N l$ for some integer $l \geq 1$. Then,
\begin{align*}
\mathbb{E}_{\mathcal{D}_n} [\|\mathcal{T}_K \hat{f} - f^*\|_{L^2(\Pi)}^2] \leq & c_1 \inf_{f \in \mathcal{V}} \mathbb{E}[(\mathcal{T}_K f - f^*)^2] + \frac{c_2 l + c_3}{n} \sup_{\mathcal{Z}: |\mathcal{Z}|=n} \log \mathcal{N}(1/n, \mathcal{T}_K \mathcal{V}, d_{\mathcal{Z}, \infty}) \\
&+ \frac{c_4 \sqrt{\log n} + c_5}{n} + \frac{c_6 n}{l} \beta((l-1)N+1),
\end{align*}
where $c_1 = 24 N$, $c_2 = 64 K^2 N$, $c_3 = 192 N \sigma^2$, $c_4 = 96 N K \sigma$, $c_5 = 32 K + 384 N \sigma^2 + 144 N \sigma$, and $c_6 = 8 K^2 / N$ are constants.
\end{theorem}

Ignoring constants and logarithmic factors that do not affect the rate, Theorem \ref{theorem: 3} states that
\begin{align*}
& \mathbb{E}_{\mathcal{D}_n} [\|\mathcal{T}_K \hat{f} - f^*\|_{L^2(\Pi)}^2] \\
& \begin{bracealign}
\lesssim \underbrace{\inf_{f \in \mathcal{V}} \mathbb{E}[(\mathcal{T}_K f - f^*)^2]}_{\text{Approximation error}} + \underbrace{\frac{l}{n} \sup_{\mathcal{Z}: |\mathcal{Z}|=n} \log \mathcal{N}(1/n, \mathcal{T}_K \mathcal{V}, d_{\mathcal{Z}, \infty})}_{\text{Generalization error}} + \underbrace{\frac{n}{l} \beta((l-1)N+1)}_{\text{Dependence error}}.
\end{bracealign}
\end{align*}
Compared with the classical decomposition in the i.i.d. case (see, e.g., \cite{schmidt2020nonparametric, kohler2021rate, jiao2023deep}), the above bound contains an additional dependence error, which arises from relating dependent and independent sequences via the sub-sample selection technique, originally introduced in \cite{kuznetsov2017generalization}. While \cite{kuznetsov2017generalization} considered the generalization and dependence errors, our analysis simultaneously accounts for three sources of errors. Although similar decompositions for dependent data have been studied in the literature, our contribution lies in several aspects. First, in contrast to \cite{feng2023over}, we use an asymmetric decomposition to establish fast rates (cf. \cite{bartlett2005local}), whereas their symmetric approach cannot give the optimal rate of convergence. Second, although \cite[Theorems 7 and 8]{ren2024statistical} also obtained fast rates using localization techniques, their bounds in expectation only hold with high probability, whereas our refined analysis provides bounds in expectation without such limitations. Finally, our setting differs in that we require matrices as RNN inputs, necessitating a block partitioning of tokens to match the input dimension. This additional structural constraint, combined with the dependence among tokens, renders the analysis more intricate. The settings in \cite{feng2023over, ren2024statistical} correspond to the special case of ours when $N=1$.

We can observe trade-offs among the three error terms. On the one hand, an appropriate choice of the hypothesis class $\mathcal{V}$ is required to balance the approximation and generalization errors, since a larger class reduces the approximation error but increases the generalization error. On the other hand, the parameter $l$ controls the trade-off between the generalization and dependence errors. Under suitable decay assumptions on the $\beta$-mixing coefficients, choosing a larger $l$ reduces the dependence error at the expense of enlarging the generalization error. Therefore, to achieve the smallest possible upper bound on the excess risk, it is crucial to select $\mathcal{V}$ and $l$ in a balanced manner.

Now let us present the detailed proof of Theorem \ref{theorem: 4}.
\begin{proof}[Proof of Theorem \ref{theorem: 4}]
By Theorem \ref{theorem: 1}, with width $W \asymp J \log J$ and depth $L \asymp I \log I$ for some $I,J\in\mathbb{N}$, 
\begin{align*}
\inf_{f \in \mathcal{V}} \mathbb{E}[(\mathcal{T}_K f - f^*)^2] \leq \inf_{f \in \mathcal{V}} \|f-f^*\|_{L^\infty([0,1]^{d_x \times N})}^2 \lesssim (J I)^{-4 \gamma / (d_x N)}.
\end{align*}
By Lemma \ref{lemma: 8}, we have
\begin{align*}
\sup_{\mathcal{Z}: |\mathcal{Z}|=n} \log \mathcal{N}(1/n, \mathcal{T}_K \mathcal{V}, d_{\mathcal{Z}, \infty}) & \lesssim W^2 L^2 \log (W L) \log n \\
& \lesssim J^2 I^2 (\log J)^2 (\log I)^2 \log (J I) \log n.
\end{align*}

$\bullet$ Case 1: if $\{x_i\}_{i=1}^n$ is geometrically $\beta$-mixing, i.e., $\beta(k) \leq \beta_0 \exp \left(-\beta_1 k^r\right)$ for some $r,\beta_0,\beta_1>0$, we choose $l \asymp (\log n)^{1/r}$ so that $\beta((l-1)N+1) \lesssim 1 / n^2$ in Theorem \ref{theorem: 3}. Then,
\begin{align*}
\mathbb{E}_{\mathcal{D}_n} [\|\mathcal{T}_K \hat{f} - f^*\|_{L^2(\Pi)}^2] \lesssim (J I)^{-4 \gamma / (d_x N)} + \frac{J^2 I^2}{n} (\log J)^2 (\log I)^2 \log (J I) (\log n)^{1+\frac{1}{r}}.
\end{align*}
For any $0 \leq \alpha \leq \frac{d_x N}{2 d_x N + 4 \gamma}$, choosing $J \asymp n^\alpha$ and $I \asymp n^{\frac{d_x N}{2 d_x N + 4 \gamma} - \alpha}$, we have
\begin{align*}
\mathbb{E}_{\mathcal{D}_n} [\|\mathcal{T}_K \hat{f} - f^*\|_{L^2(\Pi)}^2] \lesssim n^{-\frac{2 \gamma}{d_x N + 2 \gamma}} (\log n)^{6+\frac{1}{r}}.
\end{align*}

$\bullet$ Case 2: if $\{x_i\}_{i=1}^n$ is algebraically $\beta$-mixing, i.e., $\beta(k) \leq \beta_0 / k^r$ for some $r,\beta_0>0$, then Theorem \ref{theorem: 3} gives
\begin{align*}
\mathbb{E}_{\mathcal{D}_n} [\|\mathcal{T}_K \hat{f} - f^*\|_{L^2(\Pi)}^2] \lesssim (J I)^{-4 \gamma / (d_x N)} + \frac{J^2 I^2 l}{n} (\log J)^2 (\log I)^2 \log (J I) \log n + \frac{n}{l^{r+1}}.
\end{align*}
Choosing $l \asymp n^{\frac{d_x N + 4 \gamma}{(r+1) d_x N + (2r+4) \gamma}}$, $J \asymp n^\alpha$ and $I \asymp n^{\frac{r d_x N}{(2r+2) d_x N + (4r+8) \gamma} - \alpha}$ for fixed $0 \leq \alpha \leq \frac{r d_x N}{(2r+2) d_x N + (4r+8) \gamma}$, we obtain
\begin{align*}
\mathbb{E}_{\mathcal{D}_n} [\|\mathcal{T}_K \hat{f} - f^*\|_{L^2(\Pi)}^2] \lesssim n^{-\frac{2 r \gamma}{(r+1) d_x N + (2r+4) \gamma}} (\log n)^{6}.
\end{align*}

$\bullet$ Case 3: if $\{x_i\}_{i=1}^n$ is a sequence of i.i.d. random variables, then $\beta(k) = 0$ for all $k \geq 1$. Choose $l \asymp 1$, then Theorem \ref{theorem: 3} gives
\begin{align*}
\mathbb{E}_{\mathcal{D}_n} [\|\mathcal{T}_K \hat{f} - f^*\|_{L^2(\Pi)}^2] \lesssim (J I)^{-4 \gamma / (d_x N)} + \frac{J^2 I^2}{n} (\log J)^2 (\log I)^2 \log (J I) \log n.
\end{align*}
For any $0 \leq \alpha \leq \frac{d_x N}{2 d_x N + 4 \gamma}$, choosing $J \asymp n^\alpha$ and $I \asymp n^{\frac{d_x N}{2 d_x N + 4 \gamma} - \alpha}$, we have
\begin{align*}
\mathbb{E}_{\mathcal{D}_n} [\|\mathcal{T}_K \hat{f} - f^*\|_{L^2(\Pi)}^2] \lesssim n^{-\frac{2 \gamma}{d_x N + 2 \gamma}} (\log n)^{6}.
\end{align*}
By combining the analyses of the above cases, we complete the proof.
\end{proof}

\section{Conclusion}\label{sec: 6}

This paper establishes upper approximation bounds for deep ReLU RNNs. We show that RNNs, as sequence-to-sequence functions, can approximate sequences of past-dependent H\"older functions. We then present a comprehensive error analysis for regression problem with weakly dependent data using RNNs. By developing novel oracle inequalities, we prove that the RNN-based empirical risk minimizer achieves the optimal rate of convergence. Our results offer statistical guarantees on the performance of RNNs. Although RNNs have been widely applied in practical settings, the theoretical understanding of their approximation and generalization capacities remains limited. We hope that this work will inspire further research in this area.

\appendix

\section{Proofs}
In this appendix, we provide the proofs of Proposition \ref{proposition: 1}, Lemmas \ref{lemma: 8}, \ref{lemma: 10}, \ref{lemma: 5}, \ref{lemma: 4}, \ref{lemma: 7}, \ref{lemma: 6}, and Theorem \ref{theorem: 3}.

\subsection{Proof of Proposition \ref{proposition: 1}}\label{sec: A.1}
\begin{proof}[Proof of Proposition \ref{proposition: 1}]
We can parameterize $\mathcal{N}_i$ with parameters
\begin{align*}
\theta_i = (P_i, (A_{i,1}, B_{i,1}, c_{i,1}), \ldots, (A_{i,L_i}, B_{i,L_i}, c_{i,L_i}), Q_i), \quad i=1,2.
\end{align*}

(i) By appropriately adding zero rows and columns, $\mathcal{N}_1$ can also be parameterized by the parameters
\begin{align*}
\Bigg( & \begin{pmatrix}
P_1 \\
O
\end{pmatrix}, \left(
\begin{pmatrix}
A_{1,1} & \\
& O
\end{pmatrix}, 
\begin{pmatrix}
B_{1,1} & \\
& O
\end{pmatrix},
\begin{pmatrix}
c_{1,1} \\
0
\end{pmatrix}\right),
\ldots, 
\left(
\begin{pmatrix}
A_{1,L_1} & \\
& O
\end{pmatrix}, 
\begin{pmatrix}
B_{1,L_1} & \\
& O
\end{pmatrix},
\begin{pmatrix}
c_{1,L_1} \\
0
\end{pmatrix}\right), \\
&\underbrace{
\left(
O, 
\begin{pmatrix}
I & \\
& O
\end{pmatrix},
0\right), \ldots, \left(
O, 
\begin{pmatrix}
I & \\
& O
\end{pmatrix},
0\right)}_{L_2-L_1 \displaystyle \text{ times }},
(Q_1, O)\Bigg),
\end{align*}
where $I$ is the identity matrix. Hence, $\mathcal{N}_1 \in \mathcal{RNN}_{d_{x,2}, d_{y,2}}(W_2, L_2)$.

(ii) By (i), we can assume $W_1=W_2$ without loss of generality. Then, $\mathcal{N}_2 \circ \mathcal{N}_1$ can be parameterized by
\begin{align*}
(P_1, (A_{1,1}, B_{1,1}, c_{1,1}), \ldots, (A_{1,L_1}, B_{1,L_1}, c_{1,L_1}), (A_{2,1}, B_{2,1} P_2 Q_1, c_{2,1}), \ldots, (A_{2,L_2}, B_{2,L_2}, c_{2,L_2}), Q_2).
\end{align*}
Hence, we have $\mathcal{N}_2 \circ \mathcal{N}_1 \in \mathcal{RNN}_{d_{x,1}, d_{y,2}}(\max\{W_1, W_2\}, L_1+L_2)$.

(iii) By (i), we can assume that $L_1=L_2$. Then, $\mathcal{N}$ can be parameterized by the parameters $(P, (A_1, B_1, c_1), \ldots, (A_{L_1}, B_{L_1}, c_{L_1}), Q)$ with
\begin{gather*}
A_l = \begin{pmatrix}
A_{1,l} & \\
& A_{2,l}
\end{pmatrix}, \quad 
B_l = \begin{pmatrix}
B_{1,l} & \\
& B_{2,l}
\end{pmatrix}, \quad 
c_l = \begin{pmatrix}
c_{1,l} \\
c_{2,l}
\end{pmatrix}, \quad 
l = 1,\ldots,L_1, \\
P = \begin{pmatrix}
P_1 \\
P_2
\end{pmatrix}, \quad 
Q = \begin{pmatrix}
Q_1 & \\
& Q_2
\end{pmatrix}.
\end{gather*}
Hence, $\mathcal{N} \in \mathcal{RNN}_{d_{x,1}, d_{y,1}+d_{y,2}}(W_1+W_2, \max\{L_1, L_2\})$.

(iv) Replacing the matrix $Q$ in (iii) by $Q = \begin{pmatrix}
c_1 Q_1 & c_2 Q_2
\end{pmatrix}$ concludes the proof.
\end{proof}

\subsection{Proof of Lemma \ref{lemma: 5}}\label{sec: A.2}
\begin{proof}[Proof of Lemma \ref{lemma: 5}]
The proof is adapted from \cite[Lemma 3]{hoon2023minimal}. Let $\widebar{\mathcal{R}}: \mathbb{R}^{d \times N} \rightarrow \mathbb{R}^{d \times N}$ be a given modified recurrent layer with iterative computation given by $\widebar{\mathcal{R}}(X)[t]=\sigma_I (\widebar{A} \widebar{\mathcal{R}}(X)[t-1]+\widebar{B} x[t]+\widebar{c})$, and $\widebar{\mathcal{P}}: \mathbb{R}^d \rightarrow \mathbb{R}^d$ and $\widebar{\mathcal{Q}}: \mathbb{R}^d \rightarrow \mathbb{R}^d$ be given token-wise input and output embedding maps, where $\widebar{\mathcal{P}}(X)[t]= \widebar{P}x[t], \widebar{\mathcal{Q}}(X)[t]= \widebar{Q}x[t]$. We first show that for any compact subset $\Omega \subset \mathbb{R}^d$, there exist recurrent layers $\mathcal{R}_1, \mathcal{R}_2: \mathbb{R}^{(d+1) \times N} \rightarrow \mathbb{R}^{(d+1) \times N}$, and token-wise maps $\mathcal{P}: \mathbb{R}^{d \times N} \rightarrow \mathbb{R}^{(d+1) \times N}$ and $\mathcal{Q}: \mathbb{R}^{(d+1) \times N} \rightarrow \mathbb{R}^{d \times N}$ such that for $X \in \Omega^N$,
\begin{align*}
\widebar{\mathcal{Q}} \circ \widebar{\mathcal{R}} \circ \widebar{\mathcal{P}}(X) = \mathcal{Q} \circ \mathcal{R}_2 \circ \mathcal{R}_1 \circ \mathcal{P}(X).
\end{align*}
Here, the equality of two neural networks $\mathcal{N}_1 (X) = \mathcal{N}_2 (X)$ means that $\mathcal{N}_1 (X)[t] = \mathcal{N}_2 (X)[t]$ for $t = 1,\ldots,N$.

Without loss of generality, we may assume $\widebar{\mathcal{P}}$ is an identity map and $I=\{1,2,\ldots,k\}$. For fixed $z_0 > 0$, $\sigma(z_0+\delta x)= z_0+\delta x$, if $\delta$ is sufficiently small. We define the input embedding map $\mathcal{P}: \mathbb{R}^{d \times N} \rightarrow \mathbb{R}^{(d+1) \times N}$ as
\begin{align*}
\mathcal{P}(X)[t]=\begin{pmatrix}
I_d \\
O_{1,d}
\end{pmatrix} x[t] =
\begin{pmatrix}
x[t] \\
0 
\end{pmatrix}.
\end{align*}
For $\delta>0$, we first construct $\mathcal{R}_1^\delta$ as
\begin{align*}
\mathcal{R}_1^\delta(X)[t] = \sigma \left(\delta \begin{pmatrix}
\widebar{B} & \\
& 0
\end{pmatrix} x[t] + 
\delta \begin{pmatrix}
\widebar{c} \\
0
\end{pmatrix} 
+ z_0 \mathbf{1}_{d+1}\right),
\end{align*}
which implies
\begin{align*}
\mathcal{R}_1^\delta \mathcal{P}(X)[t] = \begin{pmatrix}
z_0 \mathbf{1}_d + \delta (\widebar{B} x[t] + \widebar{c}) \\
z_0
\end{pmatrix}.
\end{align*}
Next, we construct $\mathcal{R}_2^\delta$ that iteratively computes
\begin{align*}
\mathcal{R}_2^\delta(X)[t] = \sigma \left(\widetilde{A} \mathcal{R}_2^\delta(X)[t-1] + \begin{pmatrix}
\delta^{-1} I_k & \\
& I_{d+1-k}
\end{pmatrix} x[t] + 
\begin{pmatrix}
-\delta^{-1} z_0 \mathbf{1}_k \\
\mathbf{0}_{d+1-k}
\end{pmatrix}\right),
\end{align*}
where $\widetilde{A} = \begin{pmatrix}
I_k & \\ 
& \delta I_{d+1-k}
\end{pmatrix}
\begin{pmatrix}
\widebar{A} & \\ 
& 0
\end{pmatrix}
\begin{pmatrix}
I_k & & \\ 
& \delta^{-1} I_{d-k} & -\delta^{-1} \mathbf{1}_{d-k} \\ 
& & 0
\end{pmatrix}$.
Direct computation gives
\begin{align*}
\mathcal{R}_2^\delta \mathcal{R}_1^\delta \mathcal{P}(X)[1] & = \sigma \left(\begin{pmatrix}
\delta^{-1} I_k & \\
& I_{d+1-k}
\end{pmatrix} \mathcal{R}_1^\delta \mathcal{P}(X)[1] + 
\begin{pmatrix}
-\delta^{-1} z_0 \mathbf{1}_k \\
\mathbf{0}_{d+1-k}
\end{pmatrix}\right) \\
& = \sigma \begin{pmatrix}
(\widebar{B} x[1] + \widebar{c})_{1:k} \\
(z_0 \mathbf{1}_{d-k} + \delta(\widebar{B} x[1] + \widebar{c})_{k+1:d} \\
z_0
\end{pmatrix} \\
& = \begin{pmatrix}
\sigma(\widebar{B} x[1] + \widebar{c})_{1:k} \\
z_0 \mathbf{1}_{d-k} + \delta(\widebar{B} x[1] + \widebar{c})_{k+1:d} \\
z_0
\end{pmatrix} \\
& = \begin{pmatrix}
\mathcal{R}(X)[1]_{1:k} \\
z_0 \mathbf{1}_{d-k} + \delta \widebar{\mathcal{R}}(X)[1]_{k+1:d} \\
z_0
\end{pmatrix}.
\end{align*}

Assume on time $t-1$,
\begin{align*}
\mathcal{R}_2^\delta \mathcal{R}_1^\delta \mathcal{P} (X)[t-1] = \begin{pmatrix}
\widebar{\mathcal{R}}(X)[t-1]_{1:k} \\
z_0 \mathbf{1}_{d-k} + \delta \widebar{\mathcal{R}}(X)[t-1]_{k+1:d} \\
z_0
\end{pmatrix}.
\end{align*}
Direct calculation yields
\begin{align}
\label{eq: 3}
\begin{aligned}
& \begin{pmatrix}
\delta^{-1} I_k & \\
& I_{d+1-k}
\end{pmatrix} \mathcal{R}_1^\delta \mathcal{P}(X)[t] + 
\begin{pmatrix}
-\delta^{-1} z_0 \mathbf{1}_k \\
\mathbf{0}_{d+1-k}
\end{pmatrix} 
= \begin{pmatrix}
(\widebar{B} x[t] + \widebar{c})_{1:k} \\
z_0 \mathbf{1}_{d-k} + \delta(\widebar{B} x[t] + \widebar{c})_{k+1:d} \\
z_0
\end{pmatrix},
\end{aligned}
\end{align}
and
\begin{align}
\label{eq: 4}
\begin{aligned}
& \widetilde{A} \mathcal{R}_2^\delta \mathcal{R}_1^\delta \mathcal{P}(X)[t-1] \\
& = \widetilde{A} \begin{pmatrix}
\widebar{\mathcal{R}}(X)[t-1]_{1:k} \\
z_0 \mathbf{1}_{d-k} + \delta \widebar{\mathcal{R}}(X)[t-1]_{k+1:d} \\
z_0
\end{pmatrix} \\
& = \begin{pmatrix}
I_k & \\ 
& \delta I_{d+1-k}
\end{pmatrix}
\begin{pmatrix}
\widebar{A} & \\ 
& 0
\end{pmatrix}
\begin{pmatrix}
I_k & & \\ 
& \delta^{-1} I_{d-k} & -\delta^{-1} \mathbf{1}_{d-k} \\ 
& & 0
\end{pmatrix} 
\begin{pmatrix}
\widebar{\mathcal{R}}(X)[t-1]_{1:k} \\
z_0 \mathbf{1}_{d-k} + \delta \widebar{\mathcal{R}}(X)[t-1]_{k+1:d} \\
z_0
\end{pmatrix} \\
& = \begin{pmatrix}
I_k & \\ 
& \delta I_{d+1-k}
\end{pmatrix}
\begin{pmatrix}
\widebar{A} & \\ 
& 0
\end{pmatrix}
\begin{pmatrix}
\widebar{\mathcal{R}}(X)[t-1]_{1:k} \\
\widebar{\mathcal{R}}(X)[t-1]_{k+1:d} \\
0
\end{pmatrix} \\
& = \begin{pmatrix}
(\widebar{A} \widebar{\mathcal{R}}(X)[t-1])_{1:k} \\
\delta(\widebar{A} \widebar{\mathcal{R}}(X)[t-1])_{k+1:d} \\
0
\end{pmatrix}.
\end{aligned}
\end{align}
Combining (\ref{eq: 3}) and (\ref{eq: 4}), we obtain for time $t$,
\begin{align*}
\mathcal{R}_2^\delta \mathcal{R}_1^\delta \mathcal{P}(X)[t]=
\begin{pmatrix}
\widebar{\mathcal{R}}(X)[t]_{1:k} \\
z_0 \mathbf{1}_{d-k} + \delta \widebar{\mathcal{R}}(X)[t]_{k+1:d} \\
z_0
\end{pmatrix}.
\end{align*}

In the last, we define $Q^\delta = \begin{pmatrix}
\widebar{Q} & 0
\end{pmatrix}
\begin{pmatrix}
I_k & & \\
& \delta^{-1} I_{d-k} & -\delta^{-1} \mathbf{1}_{d-k} \\ 
& & 0
\end{pmatrix}$ 
to derive
\begin{align*}
\mathcal{Q}^\delta \mathcal{R}_2^\delta \mathcal{R}_1^\delta \mathcal{P}(X)[t] = \begin{pmatrix}
\widebar{Q} & 0
\end{pmatrix}
\begin{pmatrix}
\widebar{\mathcal{R}}(X)[t]_{1:k} \\
\widebar{\mathcal{R}}(X)[t]_{k+1:d} \\
0
\end{pmatrix} 
= \widebar{\mathcal{Q}} \widebar{\mathcal{R}}(X)[t], \quad t = 1,\ldots,N.
\end{align*}
Since $z_0$ is arbitrary, we choose $z_0 = \max_{t=1,\ldots,N, k=1,\ldots,d} \sup_{X \in \Omega^N} |\widebar{\mathcal{R}}(X)[t]_{k}|$ and $\delta \leq 1$ to ensure the validity of the proof above.

For any $\widebar{\mathcal{N}} \in \mathcal{MRNN}_{d_x, d_y}(W,L)$, there exist $L$ modified recurrent layers $\widebar{\mathcal{R}}_1, \ldots, \widebar{\mathcal{R}}_L$, and token-wise input and output embedding maps $\widebar{\mathcal{P}}$ and $\widebar{\mathcal{Q}}$ such that
\begin{align*}
\widebar{\mathcal{N}} = \widebar{\mathcal{Q}} \circ \widebar{\mathcal{R}}_L \circ \cdots \circ \widebar{\mathcal{R}}_1 \circ \widebar{\mathcal{P}}.
\end{align*}
Based on the previous discussion, we can find $2L$ recurrent layers $\mathcal{R}_1, \ldots, \mathcal{R}_{2L}$ and linear maps $\mathcal{P}_1, \ldots, \mathcal{P}_L, \mathcal{Q}_1, \ldots, \mathcal{Q}_L$ such that
\begin{align*}
\widebar{\mathcal{N}}(X) & = \widebar{\mathcal{Q}} \circ \widebar{\mathcal{R}}_L \circ \cdots \circ \widebar{\mathcal{R}}_1 \circ \widebar{\mathcal{P}} (X) \\
& = \widebar{\mathcal{Q}} \circ (\mathcal{Q}_L \mathcal{R}_{2L} \mathcal{R}_{2L-1} \mathcal{P}_L) \circ \cdots \circ (\mathcal{Q}_1 \mathcal{R}_2 \mathcal{R}_1 \mathcal{P}_1) \circ \widebar{\mathcal{P}} (X) \\
& = (\widebar{\mathcal{Q}}\mathcal{Q}_L) \circ \mathcal{R}_{2L} \circ (\mathcal{R}_{2L-1} \mathcal{P}_L \mathcal{Q}_{L-1}) \circ \cdots \circ (\mathcal{R}_{3} \mathcal{P}_2 \mathcal{Q}_1) \circ \mathcal{R}_2 \circ \mathcal{R}_1 \circ (\mathcal{P}_1 \widebar{\mathcal{P}}) (X) \\
& =: \mathcal{N}(X)
\end{align*}
for $X \in [0,1]^{d_x \times N}$. It is clear that $\mathcal{N} \in \mathcal{RNN}_{d_x, d_y}(W+1,2L)$, which completes the proof.
\end{proof}

\subsection{Proof of Lemma \ref{lemma: 4}}\label{sec: A.3}
\begin{proof}[Proof of Lemma \ref{lemma: 4}]
$\bullet$ Step 1: suppose $A[N-t_0+1], A[N-t_0+2], \cdots, A[N] \in \mathbb{R}^{1 \times d_x}$ are the given matrices. Then there exists an MRNN $\mathcal{N} = \mathcal{R}_N \circ \mathcal{R}_{N-1} \circ \cdots \circ \mathcal{R}_1 \circ \mathcal{P}: \mathbb{R}^{d_x \times N} \rightarrow \mathbb{R}^{(d_x+1) \times N}$ of width $d_x+1$ and depth $N$ such that (the symbol $*$ indicates a value that is irrelevant to the proof.)
\begin{align*}
\mathcal{N}(X)[t] & = \begin{pmatrix}
x[t] \\
*
\end{pmatrix} \quad \text { for } t \neq t_0, \\
\mathcal{N}(X)[t_0] & = \begin{pmatrix}
x[t_0] \\
\sum_{j=1}^{t_0} A[N-t_0+j] x[j]
\end{pmatrix}.
\end{align*}

\textit{Proof of step 1.}
We follow the construction in \cite[Lemma 6]{hoon2023minimal}. We construct an MRNN 
\begin{align*}
\mathcal{N}_n = \mathcal{R}_n \circ \mathcal{R}_{n-1} \circ \cdots \circ \mathcal{R}_1 \circ \mathcal{P},
\end{align*}
where the $l$-th modified recurrent layer $\mathcal{R}_l$ is defined as
\begin{align*}
\mathcal{R}_l (X)[t] = A_l \mathcal{R}_l (X)[t-1] + B_l x[t]
\end{align*}
with $A_l = \begin{pmatrix} 
O_{d_x, d_x} & O_{d_x, 1} \\ 
O_{1, d_x} & 1
\end{pmatrix}, B_l = \begin{pmatrix}
I_{d_x} & O_{d_x, 1} \\ 
b_l & 1
\end{pmatrix}$ for some $b_l \in \mathbb{R}^{1 \times d_x}$ to be determined later, and the input embedding map $\mathcal{P}$ is defined as
\begin{align*}
\mathcal{P} (X)[t] = \begin{pmatrix}
I_{d_x} \\ 
O_{1, d_x}
\end{pmatrix} x[t] = 
\begin{pmatrix}
x[t] \\ 
0
\end{pmatrix}.
\end{align*}
For simplicity, we denote the output of $\mathcal{N}_n$ at each time $m$ for an input sequence $X$ as
\begin{align*}
T(n, m) := \mathcal{N}_n (X)[m].
\end{align*}
Then \cite[Lemma 24]{hoon2023minimal} gives that 
\begin{align*}
T(n, m) = \begin{pmatrix}
x[m] \\
\sum_{i=1}^{\infty} \sum_{j=1}^{\infty} \binom{n+m-i-j}{n-i} b_i x[j]
\end{pmatrix},
\end{align*}
where $\binom{n}{k}$ is the binomial coefficient $\frac{n!}{k!(n-k)!}$ for $n \geq k$, and $\binom{n}{k}=0$ for $k>n$ or $n<0$ for convenience. Taking $n=N$ and $m=t_0$, we have
\begin{align*}
\mathcal{N}_N (X)[t_0] = T(N, t_0) = \begin{pmatrix}
x[t_0] \\
\sum_{i=1}^{N} \sum_{j=1}^{t_0} \binom{N+t_0-i-j}{N-i} b_i x[j]
\end{pmatrix}.
\end{align*}
Since the matrix $\Lambda_N = \left\{\binom{2n-i-j}{n-i}\right\}_{1 \leq i, j \leq N}$ has inverse $\Lambda_N^{-1} = \left\{\lambda_{i, j}\right\}_{1 \leq i, j \leq N}$ (see Lemma \ref{lemma: 1}), we set $b_i = \sum_{k=1}^N \lambda_{k, i} A[k]$ and obtain
\begin{align*}
\begin{aligned}
\sum_{i=1}^{N} \sum_{j=1}^{t_0} \binom{N+t_0-i-j}{N-i} b_i x[j] & = \sum_{i=1}^{N} \sum_{j=1}^{t_0} \binom{N+t_0-i-j}{N-i} \left(\sum_{k=1}^N \lambda_{k, i} A[k]\right) x[j] \\
& = \sum_{j=1}^{t_0} \sum_{k=1}^N \left(\sum_{i=1}^{N} \binom{N+t_0-i-j}{N-i} \lambda_{k, i}\right) A[k] x[j] \\
& = \sum_{j=1}^{t_0} \sum_{k=1}^N \delta_{k,N-t_0+j} A[k] x[j] \\
& = \sum_{j=1}^{t_0} A[N-t_0+j] x[j],
\end{aligned}
\end{align*}
where $\delta$ is the Kronecker delta function, which concludes the proof.

$\bullet$ Step 2: suppose $A_k[N-t_0+1], A_k[N-t_0+2], \cdots, A_k[N] \in \mathbb{R}^{1 \times d_x} (k=1,\ldots,W)$ are the given matrices and $c \in \mathbb{R}^{W}$ is the given bias vector. Then there exists an MRNN $\mathcal{N} = \mathcal{R}_{N+1} \circ \mathcal{R}_{N} \circ \cdots \circ \mathcal{R}_1 \circ \mathcal{P}: \mathbb{R}^{d_x \times N} \rightarrow \mathbb{R}^{(d_x+1)W \times N}$ of width $(d_x+1)W$ and depth $N+1$ such that
\begin{align*}
\mathcal{N}(X)[t_0] = \begin{pmatrix}
\sigma\left(\sum_{j=1}^{t_0} A_1[N-t_0+j] x[j] + c_1\right) \\
\sigma\left(\sum_{j=1}^{t_0} A_2[N-t_0+j] x[j] + c_2\right) \\
\vdots \\
\sigma\left(\sum_{j=1}^{t_0} A_W[N-t_0+j] x[j] + c_W\right) \\
O_{d_xW,1}
\end{pmatrix}.
\end{align*}

\textit{Proof of step 2.}
By step 1 and concatenation, we can find an MRNN $\mathcal{N}_N$ of width $(d_x+1)W$ and depth $N$ such that
\begin{align*}
\mathcal{N}_N(X)[t_0] = \begin{pmatrix}
x[t_0] \\
\sum_{j=1}^{t_0} A_1[N-t_0+j] x[j] \\
\vdots \\
x[t_0] \\
\sum_{j=1}^{t_0} A_W[N-t_0+j] x[j]
\end{pmatrix} \in \mathbb{R}^{(d_x+1)W}.
\end{align*}
We then define a modified recurrent layer $\mathcal{R}$ by $\mathcal{R}(X)[t] = \sigma(Bx[t] + \widebar{c})$, where $\widebar{c} = \begin{pmatrix}
c \\
O_{d_xW,1}
\end{pmatrix}$ and 
$B_{i,j} = \begin{cases}
1 & j = (d_x+1)i,~ i = 1,\ldots,W \\
0 & \text{otherwise}
\end{cases}$.
Then $\mathcal{N}_{N+1} := \mathcal{R} \circ \mathcal{N}_N$ is an MRNN of width $(d_x+1)W$ and depth $N+1$ such that
\begin{align*}
\mathcal{N}_{N+1}(X)[t_0] = \begin{pmatrix}
\sigma\left(\sum_{j=1}^{t_0} A_1[N-t_0+j] x[j] + c_1\right) \\
\sigma\left(\sum_{j=1}^{t_0} A_2[N-t_0+j] x[j] + c_2\right) \\
\vdots \\
\sigma\left(\sum_{j=1}^{t_0} A_W[N-t_0+j] x[j] + c_W\right) \\
O_{d_xW,1}
\end{pmatrix},
\end{align*}
which concludes the proof.

$\bullet$ Step 3: for any FNN $\widetilde{\mathcal{N}}(x[1:t_0]) \in \mathcal{FNN}_{d_x \times t_0, d_y}(W, L)$, there exists an MRNN $\mathcal{N} \in \mathcal{MRNN}_{d_x,d_y}((d_x+1)W, N+L)$ such that
\begin{align*}
\widetilde{\mathcal{N}}(x[1:t_0]) = \mathcal{N}(X)[t_0],
\end{align*}
where $X = (x[1],\ldots,x[N]) \in \mathbb{R}^{d_x \times N}$.

\textit{Proof of step 3.}
Given a feedforward neural network $\widetilde{\mathcal{N}}(x[1:t_0]) \in \mathcal{FNN}_{d_x \times t_0, d_y}(W, L)$, it can be parameterized as
\begin{align*}
((\widetilde{B}_1, \widetilde{c}_1), \ldots, (\widetilde{B}_L, \widetilde{c}_L)).
\end{align*}
Without loss of generality, assume $L \geq 3$, $\widetilde{B}_1 \in \mathbb{R}^{W \times (d_x t_0)}$, $\widetilde{B}_l \in \mathbb{R}^{W \times W}$ for $l = 2,\ldots,L-1$, $\widetilde{B}_L \in \mathbb{R}^{d_y \times W}$, $\widetilde{c}_l \in \mathbb{R}^{W}$ for $l = 1,\ldots,L-1$, and $\widetilde{c}_L \in \mathbb{R}^{d_y}$. We now directly construct an MRNN to represent the FNN $\widetilde{\mathcal{N}}$.

By step 2, there exists an MRNN $\mathcal{N}_{N+1}$ of width $(d_x+1)W$ and depth $N+1$ such that
\begin{align*}
\mathcal{N}_{N+1}(X)[t_0] = \begin{pmatrix}
\sigma (\widetilde{B}_1 \operatorname{vec}(x[1:t_0]) + \widetilde{c}_1) \\
O_{d_xW,1}
\end{pmatrix},
\end{align*}
where the operator $\operatorname{vec}(X)$ stacks the columns of the sequence $X$ into a single vector. The layers from the $2$-nd to the $L$-th of the FNN can be applied directly by defining a token-wise FNN, which is a special type of recurrent neural network that does not depend on the values from the previous time steps. Define 
\begin{align*}
\mathcal{R}_{N+l}(X)[t] := \begin{cases}
\sigma (B_l x[t] + c_l) & l = 2,\ldots,L-1 \\
B_l x[t] + c_l & l = L
\end{cases},
\end{align*}
where $B_l = \begin{pmatrix}
\widetilde{B}_l & O_{W,d_xW} \\
O_{d_xW,W} & O_{d_xW,d_xW}
\end{pmatrix}$ for $l = 2,\ldots,L-1$, $B_L = \begin{pmatrix}
\widetilde{B}_L & O_{d_y,d_xW} \\
O_{(d_x+1)W-d_y,W} & O_{(d_x+1)W-d_y,d_xW}
\end{pmatrix}$, $c_l = \begin{pmatrix}
\widetilde{c}_l \\
O_{d_xW,1}
\end{pmatrix}$ for $l = 2,\ldots,L-1$, $c_L = \begin{pmatrix}
\widetilde{c}_L \\
O_{(d_x+1)W-d_y,1}
\end{pmatrix}$, and $Q = \begin{pmatrix}
I_{d_y} & O_{d_y,(d_x+1)W-d_y}
\end{pmatrix}$. Then $\mathcal{N}_{N+L} := \mathcal{Q} \circ \mathcal{R}_{N+L} \circ \ldots \circ \mathcal{R}_{N+2} \circ \mathcal{N}_{N+1} \in \mathcal{MRNN}_{d_x,d_y}((d_x+1)W, N+L)$ satisfies
\begin{align*}
\mathcal{N}_{N+L}(X)[t_0] = \widetilde{\mathcal{N}}(x[1:t_0]) \in \mathbb{R}^{d_y},
\end{align*}
which completes the proof.
\end{proof}

\begin{lemma}
\label{lemma: 1}
The matrix $\Lambda_n = \left\{\binom{2n-i-j}{n-i}\right\}_{1 \leq i, j \leq n} \in \mathbb{R}^{n \times n}$ is invertible. Furthermore, $\Lambda_n^{-1} = \left\{\sum_{k=1}^{\min\{i,j\}} (-1)^{i+j} \binom{n-k}{n-i} \binom{n-k}{n-j} \right\}_{1 \leq i, j \leq n}$.
\end{lemma}

\begin{proof}
Note that for $i \geq j$, 
\begin{align*}
(\Lambda_n)_{i,j} & = \binom{2n-i-j}{n-i} \\
& = \sum_{k=0}^{n-i} \binom{n-i}{n-i-k} \binom{n-j}{k} \\
& = \sum_{k=0}^{n-i} \binom{n-i}{k} \binom{n-j}{k} \\
& = \sum_{k=1}^{n+1-i} \binom{n-i}{k-1} \binom{n-j}{k-1},
\end{align*}
where the second equality follows from the Chu-Vandermonde's identity. Then we have $\Lambda_n = U_n U_n^\top$ where $u_{i,j} = (U_n)_{i,j} = \begin{cases}
\binom{n-i}{j-1} & \text{ if } i+j \leq n+1 \\
0 & \text{otherwise}
\end{cases}$. Next, we prove $(U_n^{-1})_{i,j} = (-1)^{n+1+i+j} u_{n+1-i,n+1-j}$. 
\begin{itemize}[leftmargin = 12pt, labelsep = 3pt, itemsep = 0pt, topsep = 0pt]
\item If $i=j$, $\sum_{k=1}^{n} u_{i,k} (-1)^{n+1+k+j} u_{n+1-k,n+1-j} = u_{i,n+1-i} u_{i,n+1-j} = 1$.
\item If $i>j$, $\sum_{k=1}^{n} u_{i,k} (-1)^{n+1+k+j} u_{n+1-k,n+1-j} = 0$ since $u_{i,k}=0$ or $u_{n+1-k,n+1-j} = 0$.
\item If $i<j$, we have
\begin{align*}
\sum_{k=1}^{n} u_{i,k} (-1)^{n+1+k+j} u_{n+1-k,n+1-j} & = \sum_{k=n+1-j}^{n+1-i}(-1)^{n+1+k+j} \binom{n-i}{k-1} \binom{k-1}{n-j} \\
& \overset{l=k-n-1+j}{=} \sum_{l=0}^{j-i} (-1)^l \binom{n-i}{l+n-j} \binom{l+n-j}{n-j} \\
& = \sum_{l=0}^{j-i} (-1)^l \binom{j-i}{l} \binom{n-i}{n-j} \\
& = 0.
\end{align*}
\end{itemize}
It follows that 
\begin{align*}
(\Lambda_n^{-1})_{i,j} & = ((U_n^\top)^{-1} U_n^{-1})_{i,j} \\
& = \sum_{k=1}^{n} (-1)^{i+j} u_{n+1-k,n+1-i} u_{n+1-k,n+1-j} \\
& = \sum_{k=1}^{\min\{i,j\}} (-1)^{i+j} \binom{n-k}{n-i} \binom{n-k}{n-j},
\end{align*}
which concludes the proof.
\end{proof}

\subsection{Proof of Lemma \ref{lemma: 7}}\label{sec: A.4}
\begin{proof}[Proof of Lemma \ref{lemma: 7}]
By unfolding the recurrent layer $\mathcal{R}$ in time, we first prove that there exists an FNN $\widetilde{\mathcal{N}} \in \mathcal{FNN}_{t_0 W, t_0 W} ((2t_0-1) W, t_0+1)$ such that
\begin{align*}
\widetilde{\mathcal{N}} \begin{pmatrix}
x[1] \\
x[2] \\
\vdots \\
x[t_0]
\end{pmatrix} =
\begin{pmatrix}
\mathcal{R}(X)[1] \\
\mathcal{R}(X)[2] \\
\vdots \\
\mathcal{R}(X)[t_0] \\
\end{pmatrix},
\end{align*}
where $x[i] \in \mathbb{R}^W, i = 1,\ldots,t_0$.

Let $\mathcal{R}(X)[t]=\sigma(A \mathcal{R}(X)[t-1] + B x[t] + c)$ for $t = 1, \ldots, N$. We construct an FNN $\widetilde{\mathcal{N}} = \mathcal{F}_{t_0+1} \circ \mathcal{F}_{t_0} \circ \cdots \circ \mathcal{F}_1$ with parameters
\begin{align*}
\begin{gathered}
\widetilde{A}_1 = \left( \begin{array}{cccc}
B & O & \cdots & O \\
O & I & \cdots & O \\
O & -I & \cdots & O \\
\vdots & \vdots & \ddots & \vdots \\
O & O & \cdots & I \\
O & O & \cdots & -I \\
\end{array} \right), \quad 
\widetilde{b}_1 = \begin{pmatrix}
c \\
0 \\
\vdots \\
0
\end{pmatrix}, \\
\widetilde{A}_i = \left( \begin{array}{ccc:ccc:ccc}
I &  &  &  &  &  &  &  & \\
 & \ddots &  &  &  &  &  &  & \\ 
 &  & I &  &  &  &  &  & \\ \hdashline
 &  &  & I & O & O &  &  & \\
 &  &  & -I & O & O &  &  & \\ 
 &  &  & A & B & -B &  &  & \\ \hdashline
 &  &  &  &  &  & I &  & \\
 &  &  &  &  &  &  & \ddots & \\
 &  &  &  &  &  &  &  & I \\
\end{array} \right), \quad 
\widetilde{b}_i = \begin{pmatrix}
0 \\
\vdots \\
0 \\
c \\
0\\
\vdots \\
0
\end{pmatrix}, \quad i = 2, \ldots, t_0, \\
\widetilde{A}_{t_0+1} = \left( \begin{array}{cccccccc}
I & -I &  &  &  &  &  & \\
 &  & I & -I &  &  &  & \\
 &  &  &  & \ddots &  &  &  \\
 &  &  &  &  & I & -I & \\
 &  &  &  &  &  &  & I \\
\end{array} \right), \quad 
\widetilde{b}_{t_0+1} = 0,
\end{gathered}
\end{align*}
where in the block matrices each block is of size $W \times W$, $\widetilde{A}_1$ has $2t_0 - 1$ rows and $t_0$ columns of blocks, $\widetilde{A}_i$ has $2i - 4$ identity matrices in the upper left corner and $2t_0 - 2i$ identity matrices in the lower right corner, and $\widetilde{A}_{t_0+1}$ has $t_0$ rows and $2t_0 - 1$ columns of blocks. Direct calculation yields
\begin{align*}
\mathcal{F}_i \circ \cdots \circ \mathcal{F}_1 \begin{pmatrix}
x[1] \\
x[2] \\
\vdots \\
x[t_0]
\end{pmatrix} = \begin{pmatrix}
\sigma(\mathcal{R}(X)[1]) \\
\sigma(-\mathcal{R}(X)[1]) \\
\vdots \\
\sigma(\mathcal{R}(X)[i-1]) \\
\sigma(-\mathcal{R}(X)[i-1]) \\
\mathcal{R}(X)[i] \\
\sigma(x[i+1]) \\
\sigma(-x[i+1]) \\
\vdots \\
\sigma(x[t_0]) \\
\sigma(-x[t_0])
\end{pmatrix}, \quad i = 1,\ldots,t_0.
\end{align*}
Hence,
\begin{align*}
\widetilde{\mathcal{N}} \begin{pmatrix}
x[1] \\
x[2] \\
\vdots \\
x[t_0]
\end{pmatrix} & = \mathcal{F}_{t_0+1} \circ \mathcal{F}_{t_0} \circ \cdots \circ \mathcal{F}_1 \begin{pmatrix}
x[1] \\
x[2] \\
\vdots \\
x[t_0]
\end{pmatrix} \\
& = 
\mathcal{F}_{t_0+1} \begin{pmatrix}
\sigma(\mathcal{R}(X)[1]) \\
\sigma(-\mathcal{R}(X)[1]) \\
\vdots \\
\sigma(\mathcal{R}(X)[t_0-1]) \\
\sigma(-\mathcal{R}(X)[t_0-1]) \\
\mathcal{R}(X)[t_0]
\end{pmatrix} \\
& = \widetilde{A}_{t_0+1} \begin{pmatrix}
\sigma(\mathcal{R}(X)[1]) \\
\sigma(-\mathcal{R}(X)[1]) \\
\vdots \\
\sigma(\mathcal{R}(X)[t_0-1]) \\
\sigma(-\mathcal{R}(X)[t_0-1]) \\
\mathcal{R}(X)[t_0]
\end{pmatrix} + \widetilde{b}_{t_0+1} \\
& = \begin{pmatrix}
\mathcal{R}(X)[1] \\
\mathcal{R}(X)[2] \\
\vdots \\
\mathcal{R}(X)[t_0] \\
\end{pmatrix},
\end{align*}
where we use the identity $\sigma(x)-\sigma(-x) = x$.

Then for any $\mathcal{N} = \mathcal{Q} \circ \mathcal{R}_L \circ \cdots \circ \mathcal{R}_1 \circ \mathcal{P} \in \mathcal{RNN}_{d_x,d_y}(W,L)$, we can find $\widetilde{\mathcal{N}}_1, \cdots, \widetilde{\mathcal{N}}_L \in \mathcal{FNN}_{t_0 W, t_0 W} ((2t_0-1) W, t_0+1)$ and a linear mapping $\widetilde{\mathcal{P}}$ such that
\begin{align*}
\widetilde{\mathcal{N}}_L \circ \cdots \circ \widetilde{\mathcal{N}}_2 \circ \widetilde{\mathcal{N}}_1 \circ \widetilde{\mathcal{P}}(\operatorname{vec}(x[1:t_0])) & = \widetilde{\mathcal{N}}_L \circ \cdots \circ \widetilde{\mathcal{N}}_2 \circ \widetilde{\mathcal{N}}_1 \begin{pmatrix}
\mathcal{P}(X)[1] \\
\mathcal{P}(X)[2] \\
\vdots \\
\mathcal{P}(X)[t_0] \\
\end{pmatrix} \\
& = \widetilde{\mathcal{N}}_L \circ \cdots \circ \widetilde{\mathcal{N}}_2 \begin{pmatrix}
\mathcal{R}_1 \mathcal{P}(X)[1] \\
\mathcal{R}_1 \mathcal{P}(X)[2] \\
\vdots \\
\mathcal{R}_1 \mathcal{P}(X)[t_0] \\
\end{pmatrix} \\
& \cdots \\
& = \begin{pmatrix}
\mathcal{R}_L \cdots \mathcal{R}_1 \mathcal{P}(X)[1] \\
\mathcal{R}_L \cdots \mathcal{R}_1 \mathcal{P}(X)[2] \\
\vdots \\
\mathcal{R}_L \cdots \mathcal{R}_1 \mathcal{P}(X)[t_0] \\
\end{pmatrix}.
\end{align*}
In the last, another linear mapping $\widetilde{\mathcal{Q}}$ applies $\mathcal{Q}$ to the last segment to obtain
\begin{align*}
(\widetilde{\mathcal{Q}} \widetilde{\mathcal{N}}_L) \circ \widetilde{\mathcal{N}}_{L-1} \cdots \circ \widetilde{\mathcal{N}}_2 \circ (\widetilde{\mathcal{N}}_1 \widetilde{\mathcal{P}})(\operatorname{vec}(x[1:t_0])) = \mathcal{Q} \mathcal{R}_L \cdots \mathcal{R}_1 \mathcal{P}(X)[t_0] = \mathcal{N}(X)[t_0].
\end{align*}
$(\widetilde{\mathcal{Q}} \widetilde{\mathcal{N}}_L) \circ \widetilde{\mathcal{N}}_{L-1} \cdots \circ \widetilde{\mathcal{N}}_2 \circ (\widetilde{\mathcal{N}}_1 \widetilde{\mathcal{P}}) \in \mathcal{FNN}_{d_x \times t_0, d_y}((2t_0-1) W, t_0 L)$ concludes the proof.
\end{proof}

\subsection{Proof of Lemma \ref{lemma: 6}}\label{sec: A.5}
\begin{proof}[Proof of Lemma \ref{lemma: 6}]
A series of studies have indicated that feedforward neural networks can effectively approximate H\"older smooth functions \cite{yarotsky2017error, lu2021deep, shen2022optimal, jiao2023deep}. For instance, \cite[Corollary 3.1]{jiao2023deep} gives that for any $k \in \{1, \ldots, d_y\}$ and $I, J \in \mathbb{N}$, there exists a neural network $\widetilde{\mathcal{N}}_k \in \mathcal{FNN}_{d_x \times t_0,1}(W, L)$ with width
\begin{align*}
W=38(\lfloor\gamma\rfloor+1)^2 3^{d_x t_0} (d_x t_0)^{\lfloor\gamma\rfloor+1} J \lceil\log_2(8 J)\rceil
\end{align*}
and depth 
\begin{align*}
L=21(\lfloor\gamma\rfloor + 1)^2 I \lceil\log_2(8 I)\rceil+2 d_x t_0,
\end{align*}
such that
\begin{align*}
\sup_{x[1:t_0] \in [0,1]^{d_x \times t_0}} |f_k(x[1:t_0])-\widetilde{\mathcal{N}}_k (x[1:t_0])| \leq 19 K(\lfloor\gamma\rfloor+1)^2 (d_x t_0)^{\lfloor\gamma\rfloor+(\gamma \vee 1) / 2}(J I)^{-2 \gamma / (d_x t_0)}.
\end{align*}
By concatenation of neural networks, there exists an FNN $\widetilde{\mathcal{N}} = (\widetilde{\mathcal{N}}_1, \ldots, \widetilde{\mathcal{N}}_{d_y})^\top \in \mathcal{FNN}_{d_x \times t_0,d_y}(W, L)$ with width 
\begin{align*}
W=38(\lfloor\gamma\rfloor+1)^2 3^{d_x t_0} (d_x t_0)^{\lfloor\gamma\rfloor+1} d_y J \lceil\log_2(8 J)\rceil
\end{align*}
and depth 
\begin{align*}
L=21(\lfloor\gamma\rfloor + 1)^2 I \lceil\log_2(8 I)\rceil+2 d_x t_0,
\end{align*}
such that
\begin{align*}
\sup_{x[1:t_0] \in [0,1]^{d_x \times t_0}}  \|f(x[1:t_0])-\widetilde{\mathcal{N}} (x[1:t_0])\|_{\infty} \leq 19 K(\lfloor\gamma\rfloor+1)^2 (d_x t_0)^{\lfloor\gamma\rfloor+(\gamma \vee 1) / 2}(J I)^{-2 \gamma / (d_x t_0)}.
\end{align*}

By Proposition \ref{proposition: 3}, there exists an RNN $\mathcal{N} \in \mathcal{RNN}_{d_x, d_y}(W,L)$ with width 
\begin{align*}
W & = 38(\lfloor\gamma\rfloor+1)^2 3^{d_x t_0} (d_x+1) (d_x t_0)^{\lfloor\gamma\rfloor+1} d_y J \lceil\log_2(8 J)\rceil + 1 \\
& \leq 152 (\lfloor\gamma\rfloor+1)^2 3^{d_x t_0} d_x^{\lfloor\gamma\rfloor+2} d_y t_0^{\lfloor\gamma\rfloor+1} J \lceil\log_2(8 J)\rceil
\end{align*}
and depth 
\begin{align*}
L = 42(\lfloor\gamma\rfloor + 1)^2 I \lceil\log_2(8 I)\rceil + 4 d_x t_0 + 2N
\end{align*}
such that
\begin{align*}
\widetilde{\mathcal{N}}(x[1:t_0]) = \mathcal{N}(X)[t_0], \quad X \in [0,1]^{d_x \times N},
\end{align*}
and 
\begin{align*}
\sup_{X \in [0,1]^{d_x \times N}}  \|f(x[1:t_0])-\mathcal{N}(X)[t_0]\|_{\infty} \leq 19 K(\lfloor\gamma\rfloor+1)^2 (d_x t_0)^{\lfloor\gamma\rfloor+(\gamma \vee 1) / 2}(J I)^{-2 \gamma / (d_x t_0)},
\end{align*}
which completes the proof.
\end{proof}

\subsection{Proof of Lemma \ref{lemma: 8}}\label{sec: A.6}
\begin{definition}[Pseudo-dimension]
Let $\mathcal{H}$ be a class of real-valued functions defined on $\Omega$. The pseudo-dimension of $\mathcal{H}$, denoted by $\operatorname{Pdim}(\mathcal{H})$, is the largest integer $N$ for which there exist points $x_1, \ldots, x_N \in \Omega$ and constants $c_1, \ldots, c_N \in \mathbb{R}$ such that
\begin{align*}
|\{\operatorname{sgn}(h(x_1)-c_1), \ldots, \operatorname{sgn}(h(x_N)-c_N): h \in \mathcal{H}\}| = 2^N.
\end{align*}
\end{definition}

\begin{proof}[Proof of Lemma \ref{lemma: 8}]
By Proposition \ref{proposition: 3}, we have $\mathcal{V}_{t_0} \subseteq \mathcal{FNN}_{d_x \times t_0, 1}((2t_0-1) W, t_0 L)$. We next bound the covering number of $\mathcal{F} := \mathcal{FNN}_{d_x \times t_0, 1}((2t_0-1) W, t_0 L)$. By Theorem 7 of \cite{bartlett2019nearly}, we have
\begin{align*}
\operatorname{Pdim}(\mathcal{F}) & \leq c_1 ((2t_0-1) W)^2 (t_0 L)^2 \log ((2t_0-1) W \cdot t_0 L) \\
& \leq c_2 t_0^4 W^2 L^2 \log (t_0 W L).
\end{align*}

(1) By Theorem 12.2 of \cite{anthony2009neural}, 
\begin{align*}
\log \mathcal{N} (\delta, \mathcal{T}_K \mathcal{V}_{t_0}, d_{\mathcal{Z}, \infty}) & \leq \log \mathcal{N} (\delta, \mathcal{T}_K \mathcal{F}, d_{\mathcal{Z}, \infty}) \\
& \leq \operatorname{Pdim}(\mathcal{F}) \log \left(\frac{e n K}{\delta \operatorname{Pdim}(\mathcal{F})}\right) \\
& \leq c_3 t_0^4 W^2 L^2 \log (t_0 W L) \log (K n / \delta),
\end{align*}
where $c_3$ is a universal constant. Taking the supremum over all $\mathcal{Z}$ completes the proof.

(2) By Theorem 5.11 of \cite{zhang2023mathematical}, 
\begin{align*}
\log \mathcal{N} (\delta, \mathcal{T}_K \mathcal{V}_{t_0}, d_p) & \leq \log \mathcal{N} (\delta, \mathcal{T}_K \mathcal{F}, d_p) \\
& \leq 1 + \log (\operatorname{Pdim}(\mathcal{F}) + 1) + \operatorname{Pdim}(\mathcal{F}) \log (2 e K^p / \delta^p) \\
& \leq c_4 p t_0^4 W^2 L^2 \log (t_0 W L) \log (K / \delta),
\end{align*}
where $c_4$ is a universal constant.

\end{proof}

\subsection{Proof of Lemma \ref{lemma: 10}}\label{sec: A.7}
\begin{definition}[Packing number]
Let $\rho$ be a pseudo-metric on $\mathcal{M}$ and $S \subseteq \mathcal{M}$. For any $\delta>0$, a set $A \subseteq S$ is called a $\delta$-packing of $S$ if for any distinct $x,y \in A$ we have $\rho(x,y) > \delta$. The $\delta$-packing number of $S$, denoted by $\mathcal{M}(\delta, S, \rho)$, is the maximum cardinality of any $\delta$-packing of $S$.
\end{definition}

\begin{proof}[Proof of Lemma \ref{lemma: 10}]
Let $\{h_i\}$ be a $(4 \delta)$-packing of $\mathcal{H}_{d,1}^\gamma ([0,1]^{d}, 1)$, so that
\begin{align*}
|\{h_i\}| = \mathcal{M}(4 \delta, \mathcal{H}_{d,1}^\gamma ([0,1]^{d}, 1), d_p).
\end{align*}
By definition, $\|h_i - h_j\|_{L^p(\lambda)} > 4 \delta$ whenever $i \neq j$. By assumption, for each $h_i$, there exists $f_i \in \mathcal{F}$ such that $\|f_i - h_i\|_{L^p(\lambda)} \leq \delta$. It follows from the triangle inequality that
\begin{align*}
\|f_i - f_j\|_{L^p(\lambda)} 
& \geq \|h_i - h_j\|_{L^p(\lambda)} - \|f_i - h_i\|_{L^p(\lambda)} - \|f_j - h_j\|_{L^p(\lambda)} \\
& > 4 \delta - \delta - \delta = 2 \delta.
\end{align*}
Hence $\{f_i\}$ forms a $(2 \delta)$-packing of $\mathcal{F}$, which implies
\begin{align}\label{eq: 26}
\mathcal{M}(2 \delta, \mathcal{F}, d_p) 
\geq |\{h_i\}| 
= \mathcal{M}(4 \delta, \mathcal{H}_{d,1}^\gamma ([0,1]^{d}, 1), d_p).
\end{align}

On the one hand, the covering number upper bounds the packing number (see, e.g., \cite[Theorem 5.2]{zhang2023mathematical}),
\begin{align}\label{eq: 27}
\mathcal{N}(\delta, \mathcal{F}, d_p) \geq \mathcal{M}(2 \delta, \mathcal{F}, d_p).
\end{align}
On the other hand, since $\|f\|_{L^1(\lambda)} \leq \|f\|_{L^p(\lambda)}$ for any measurable $f$, the packing number in $L^p$ is lower bounded by that in $L^1$. Together with known lower bounds on the packing number of the H\"older class \cite{birman1967piecewise, carl1981entropy, triebel1975interpolation}, we obtain
\begin{align}\label{eq: 28}
\log \mathcal{M}(4 \delta, \mathcal{H}_{d,1}^\gamma ([0,1]^{d}, 1), d_p)
\geq \log \mathcal{M}(4 \delta, \mathcal{H}_{d,1}^\gamma ([0,1]^{d}, 1), d_1)
\geq c_{d, \gamma} \, \delta^{-d/\gamma}.
\end{align}
Combining (\ref{eq: 26}), (\ref{eq: 27}) and (\ref{eq: 28}) yields the desired result.
\end{proof}

\subsection{Proof of Theorem \ref{theorem: 3}}\label{sec: A.8}

To prove Theorem \ref{theorem: 3}, we use the sub-sample selection technique, which enables us to relate mixing sequences to independent ones. Given a stationary sequence $\mathcal{X} = \{x_i\}_{i=1}^n$ of $n$ tokens, we partition it into blocks of tokens such that each block has size $d_x \times N$ and can be the input to the target function and the neural network.

First, starting from the initial token, we divide every consecutive $N$ tokens into one block, discarding any remaining terms at the end. The resulting blocks are collected as $\mathcal{B} = \{B_1, \ldots, B_{\tilde{n}}\}$, where $B_i = (x_{(i-1)N+1}, \ldots, x_{iN})$, and $\tilde{n}=\lfloor n/N \rfloor$ is the possible maximum number of blocks. Recall that $\Pi$ denotes the joint distribution of the predictors. Under the assumption of stationarity, each block $B_i$ follows the same distribution $\Pi$. We further select sub-blocks, denoted by
\begin{align*}
\mathcal{B}^{(a)} = \{B_{a+kl}: k = 0, 1, \ldots, \lfloor\tilde{n}/l\rfloor - 1\}, \quad a = 1, \ldots, l,
\end{align*}
where $l \in \mathbb{N}$ is a parameter to be specified later.

Second, we adopt a sliding window partitioning, where each block is formed by any $N$ consecutive tokens. The resulting blocks are collected as $\mathcal{S} = \{S_N, \ldots, S_n\}$, where $S_t = (x_{t-N+1}, \ldots, x_t)$. Under the assumption of stationarity, each block $S_t$ follows the same distribution $\Pi$.

Third, in place of the dependent blocks, we construct another collection of blocks $\widetilde{\mathcal{B}} = \{\widetilde{B}_1, \ldots, \widetilde{B}_m\}$ consisting of $m$ independent blocks. Each block $\widetilde{B}_i$ contains $N$ tokens and follows the same distribution $\Pi$ as $B_i$, and the blocks are mutually independent.

For convenience, we define the mean squared error at the population level as
\begin{align*}
\mathcal{L}(f) & = \mathbb{E}_{(x_1, \ldots, x_N) \sim \Pi} [(f(x_1, \ldots, x_N) - f^*(x_1, \ldots, x_N))^2].
\end{align*}
Given a set of blocks $\mathcal{C}$ (which may be chosen as $\mathcal{B}$ or any of the other block collections defined above), the empirical counterpart is given by
\begin{align*}
\mathcal{L}_{|\mathcal{C}|}(f; \mathcal{C}) = \frac{1}{|\mathcal{C}|} \sum_{(x_1, \ldots, x_N) \in \mathcal{C}} (f(x_1, \ldots, x_N) - f^*(x_1, \ldots, x_N))^2,
\end{align*}
where $|\mathcal{C}|$ denotes the cardinality of the block set $\mathcal{C}$.

\begin{proof}[Proof of Theorem \ref{theorem: 3}]

To bound $\mathbb{E}[\|\mathcal{T}_K \hat{f} - f^*\|_{L^2(\Pi)}^2] = \mathbb{E}[\mathcal{L}(\mathcal{T}_K \hat{f})]$, technically we divide our proof into 5 steps. In Step 4 we decompose the error into two terms. In Step 1 and Step 2 we bound the first term. In Step 3 we bound the second term. Finally, in Step 5 we combine all components to prove the claim.

$\bullet$ Step 1: Let $\widetilde{\mathcal{B}}$ be a collection of $m$ independent blocks. For any $\delta \in (0,1)$, we have
\begin{align}
\label{eq: 16}
\mathbb{E}_{\widetilde{\mathcal{B}}} \left[\sup_{f \in \mathcal{V}} \mathcal{L}(\mathcal{T}_K f) - 3 \mathcal{L}_m (\mathcal{T}_K f; \widetilde{\mathcal{B}})\right] \leq \frac{16 K^2}{m} \sup_{|\widetilde{\mathcal{B}}|=m} \log \mathcal{N}(\delta, \mathcal{T}_K \mathcal{V}, d_{\widetilde{\mathcal{B}}, 1}) + 32 K \delta,
\end{align}
where 
\begin{align*}
d_{\widetilde{\mathcal{B}}, 1}(f,g) := \frac{1}{m} \sum_{i=1}^m |f(\widetilde{B}_i)-g(\widetilde{B}_i)|.
\end{align*}

\textit{Proof of step 1.}
Let us denote
\begin{align*}
\mathcal{H} = \{h = (\mathcal{T}_K f - f^*)^2: f \in \mathcal{V}\}.
\end{align*}
Since $h = (\mathcal{T}_K f - f^*)^2 \leq 2 (\mathcal{T}_K f)^2 + 2(f^*)^2 \leq 4 K^2$ for any $h \in \mathcal{H}$, we have
\begin{align*}
\begin{aligned}
& \mathbb{E}_{\widetilde{\mathcal{B}}} \left[\sup_{f \in \mathcal{V}} \mathcal{L}(\mathcal{T}_K f) - 3 \mathcal{L}_m (\mathcal{T}_K f; \widetilde{\mathcal{B}})\right] \\
& = \mathbb{E}_{\widetilde{\mathcal{B}}} \left[\sup_{f \in \mathcal{V}} \mathbb{E}[(\mathcal{T}_K f-f^*)^2] - \frac{3}{m} \sum_{i=1}^m \left(\mathcal{T}_K f(\widetilde{B}_i)-f^*(\widetilde{B}_i)\right)^2\right] \\
& = \mathbb{E}_{\widetilde{\mathcal{B}}} \left[\sup_{h \in \mathcal{H}} \mathbb{E}[h] - \frac{3}{m} \sum_{i=1}^m h(\widetilde{B}_i)\right] \\
& = \mathbb{E}_{\widetilde{\mathcal{B}}} \left[\sup_{h \in \mathcal{H}} 2 \mathbb{E}[h]-\mathbb{E}[h]-\frac{2}{m} \sum_{i=1}^m h(\widetilde{B}_i)-\frac{1}{m} \sum_{i=1}^m h(\widetilde{B}_i)\right] \\
& \leq \mathbb{E}_{\widetilde{\mathcal{B}}} \left[\sup_{h \in \mathcal{H}} 2 \mathbb{E}[h]-\frac{1}{4 K^2} \mathbb{E}[h^2]-\frac{2}{m} \sum_{i=1}^m h(\widetilde{B}_i)-\frac{1}{4 K^2 m} \sum_{i=1}^m h^2(\widetilde{B}_i)\right].
\end{aligned}
\end{align*}
We can then use the standard symmetrization technique to bound it by Rademacher complexity. We introduce a ghost dataset $\widetilde{\mathcal{B}}^{\prime}=\{\widetilde{B}_i^{\prime}\}_{i=1}^m$ drawn i.i.d. from the same distribution $\Pi$ as in $\widetilde{\mathcal{B}}$ and let $\tau=\{\tau_i\}_{i=1}^m$ be a sequence of i.i.d. Rademacher variables independent of $\widetilde{\mathcal{B}}$ and $\widetilde{\mathcal{B}}^{\prime}$. Then we have
\begin{align*}
\begin{aligned}
& \mathbb{E}_{\widetilde{\mathcal{B}}} \left[\sup_{h \in \mathcal{H}} 2 \mathbb{E}[h]-\frac{1}{4 K^2} \mathbb{E}[h^2]-\frac{2}{m} \sum_{i=1}^m h(\widetilde{B}_i)-\frac{1}{4 K^2 m} \sum_{i=1}^m h^2(\widetilde{B}_i)\right] \\
= & \mathbb{E}_{\widetilde{\mathcal{B}}} \left[\sup_{h \in \mathcal{H}} \mathbb{E}_{\widetilde{\mathcal{B}}^{\prime}} \left[\frac{2}{m} \sum_{i=1}^m h(\widetilde{B}_i^{\prime})-\frac{1}{4 K^2 m} \sum_{i=1}^m h^2(\widetilde{B}_i^{\prime})\right] -\frac{2}{m} \sum_{i=1}^m h(\widetilde{B}_i) -\frac{1}{4 K^2 m} \sum_{i=1}^m h^2(\widetilde{B}_i)\right] \\
\leq & \mathbb{E}_{\widetilde{\mathcal{B}}, \widetilde{\mathcal{B}}^{\prime}} \left[\sup_{h \in \mathcal{H}} \frac{2}{m} \sum_{i=1}^m h(\widetilde{B}_i^{\prime}) -\frac{2}{m} \sum_{i=1}^m h(\widetilde{B}_i) -\frac{1}{4 K^2 m} \sum_{i=1}^m h^2(\widetilde{B}_i^{\prime}) -\frac{1}{4 K^2 m} \sum_{i=1}^m h^2(\widetilde{B}_i)\right] \\
= & \mathbb{E}_{\widetilde{\mathcal{B}}, \widetilde{\mathcal{B}}^{\prime}, \tau} \left[\sup_{h \in \mathcal{H}} \frac{2}{m} \sum_{i=1}^m \tau_i \left(h(\widetilde{B}_i^{\prime})-h(\widetilde{B}_i)\right) -\frac{1}{4 K^2 m} \sum_{i=1}^m \left(h^2(\widetilde{B}_i^{\prime})+h^2(\widetilde{B}_i)\right)\right] \\
\leq & \mathbb{E}_{\widetilde{\mathcal{B}}, \widetilde{\mathcal{B}}^{\prime}, \tau} \left[\sup_{h \in \mathcal{H}} \left(\frac{2}{m} \sum_{i=1}^m \tau_i h(\widetilde{B}_i^{\prime})-\frac{1}{4 K^2 m} \sum_{i=1}^m h^2(\widetilde{B}_i^{\prime})\right) + \sup_{h \in \mathcal{H}} \left(\frac{2}{m} \sum_{i=1}^m (-\tau_i) h(\widetilde{B}_i)-\frac{1}{4 K^2 m} \sum_{i=1}^m h^2(\widetilde{B}_i)\right)\right] \\
= & \mathbb{E}_{\widetilde{\mathcal{B}}, \tau} \left[\sup_{h \in \mathcal{H}} \frac{4}{m} \sum_{i=1}^m \tau_i h(\widetilde{B}_i)-\frac{1}{2 K^2 m} \sum_{i=1}^m h^2(\widetilde{B}_i)\right],
\end{aligned}
\end{align*}
where the second equality follows from the fact that randomly interchanging the corresponding blocks of $\widetilde{\mathcal{B}}$ and $\widetilde{\mathcal{B}}^{\prime}$ doesn't affect the joint distribution of $\widetilde{\mathcal{B}}, \widetilde{\mathcal{B}}^{\prime}$ and the summation $\sum_{i=1}^m (h^2(\widetilde{B}_i^{\prime})+h^2(\widetilde{B}_i))$, and the last equality is due to $\widetilde{B}_i$ and $\widetilde{B}_i^{\prime}$ have the same distribution and $\tau_i$ and $-\tau_i$ have the same distribution.

For any fixed $\widetilde{\mathcal{B}}$, we discretize $\mathcal{H}$ with respect to the pseudo-metric $d_{\widetilde{\mathcal{B}}, 1}$. Let $\mathcal{H}_\delta(\widetilde{\mathcal{B}})$ be a $\delta$-covering of $\mathcal{H}$ with minimal cardinality under the pseudo-metric $d_{\widetilde{\mathcal{B}}, 1}$, then for any $h \in \mathcal{H}$, there exists $g \in \mathcal{H}_\delta(\widetilde{\mathcal{B}})$ such that $\frac{1}{m} \sum_{i=1}^m |h(\widetilde{B}_i)-g(\widetilde{B}_i)| \leq \delta$, which implies
\begin{align*}
\begin{aligned}
\frac{1}{m} \sum_{i=1}^m \tau_i h(\widetilde{B}_i) & \leq \frac{1}{m} \sum_{i=1}^m \tau_i g(\widetilde{B}_i) + \frac{1}{m} \sum_{i=1}^m |\tau_i| \cdot |h(\widetilde{B}_i)-g(\widetilde{B}_i)| \\
& \leq \frac{1}{m} \sum_{i=1}^m \tau_i g(\widetilde{B}_i) + \delta
\end{aligned}
\end{align*}
and, since $|h(\widetilde{B}_i)|,|g(\widetilde{B}_i)| \leq 4 K^2$,
\begin{align*}
\begin{aligned}
\frac{1}{m} \sum_{i=1}^m h^2(\widetilde{B}_i) & =\frac{1}{m} \sum_{i=1}^m g^2(\widetilde{B}_i) + \frac{1}{m} \sum_{i=1}^m \left(h(\widetilde{B}_i) + g(\widetilde{B}_i)\right)\left(h(\widetilde{B}_i)-g(\widetilde{B}_i)\right) \\
& \geq \frac{1}{m} \sum_{i=1}^m g^2(\widetilde{B}_i) - 8 K^2 \frac{1}{m} \sum_{i=1}^m |h(\widetilde{B}_i)-g(\widetilde{B}_i)| \\
& \geq \frac{1}{m} \sum_{i=1}^m g^2(\widetilde{B}_i) - 8 K^2 \delta.
\end{aligned}
\end{align*}
Therefore,
\begin{align*}
\begin{aligned}
\mathbb{E}_{\widetilde{\mathcal{B}}} \left[\sup_{f \in \mathcal{V}} \mathcal{L}(\mathcal{T}_K f) - 3 \mathcal{L}_m (\mathcal{T}_K f; \widetilde{\mathcal{B}})\right] & \leq \mathbb{E}_{\widetilde{\mathcal{B}}, \tau} \left[\sup_{h \in \mathcal{H}} \frac{4}{m} \sum_{i=1}^m \tau_i h(\widetilde{B}_i) - \frac{1}{2 K^2 m} \sum_{i=1}^m h^2(\widetilde{B}_i)\right] \\
& \leq 4 \mathbb{E}_{\widetilde{\mathcal{B}}, \tau} \left[\sup_{g \in \mathcal{H}_\delta(\widetilde{\mathcal{B}})} \frac{1}{m} \sum_{i=1}^m \tau_i g(\widetilde{B}_i) - \frac{1}{8 K^2 m} \sum_{i=1}^m g^2(\widetilde{B}_i)\right] + 8 \delta.
\end{aligned}
\end{align*}

For fixed $\widetilde{\mathcal{B}}$ and any $\lambda>0$, we have
\begin{align*}
\begin{aligned}
& \exp \left(\lambda \mathbb{E}_\tau \left[\sup_{g \in \mathcal{H}_\delta(\widetilde{\mathcal{B}})} \frac{1}{m} \sum_{i=1}^m \tau_i g(\widetilde{B}_i) - \frac{1}{8 K^2 m} \sum_{i=1}^m g^2(\widetilde{B}_i)\right]\right) \\
\leq & \mathbb{E}_\tau \left[\exp \left(\lambda \sup_{g \in \mathcal{H}_\delta(\widetilde{\mathcal{B}})} \frac{1}{m} \sum_{i=1}^m \tau_i g(\widetilde{B}_i) - \frac{1}{8 K^2 m} \sum_{i=1}^m g^2(\widetilde{B}_i)\right)\right] \\
\leq & \sum_{g \in \mathcal{H}_\delta(\widetilde{\mathcal{B}})} \mathbb{E}_\tau \left[\exp \left(\frac{\lambda}{m} \sum_{i=1}^m \tau_i g(\widetilde{B}_i) - \frac{\lambda}{8 K^2 m} \sum_{i=1}^m g^2(\widetilde{B}_i)\right)\right] \\
= & \sum_{g \in \mathcal{H}_\delta(\widetilde{\mathcal{B}})} \prod_{i=1}^m \mathbb{E}_{\tau_i}\left[\exp \left(\frac{\lambda}{m} \tau_i g(\widetilde{B}_i) - \frac{\lambda}{8 K^2 m} g^2(\widetilde{B}_i)\right)\right] \\
\leq & \sum_{g \in \mathcal{H}_\delta(\widetilde{\mathcal{B}})} \prod_{i=1}^m \exp \left(\frac{\lambda^2}{2 m^2} g^2(\widetilde{B}_i) - \frac{\lambda}{8 K^2 m} g^2(\widetilde{B}_i)\right),
\end{aligned}
\end{align*}
where we use $\mathbb{E}_{\tau_i} [\exp (\frac{\lambda}{m} \tau_i g(\widetilde{B}_i))] = \cosh (\frac{\lambda}{m} g(\widetilde{B}_i)) \leq \exp (\frac{\lambda^2}{2 m^2} g^2(\widetilde{B}_i))$ in the last inequality. Taking $\lambda = \frac{m}{4 K^2}$ in the above inequality, we obtain
\begin{align*}
\mathbb{E}_\tau \left[\sup_{g \in \mathcal{H}_\delta(\widetilde{\mathcal{B}})} \frac{1}{m} \sum_{i=1}^m \tau_i g(\widetilde{B}_i) - \frac{1}{8 K^2 m} \sum_{i=1}^m g^2(\widetilde{B}_i)\right] \leq \frac{1}{\lambda} \log |\mathcal{H}_\delta(\widetilde{\mathcal{B}})| = \frac{4 K^2}{m} \log \mathcal{N}(\delta, \mathcal{H}, d_{\widetilde{\mathcal{B}}, 1}),
\end{align*}
indicating
\begin{align*}
\mathbb{E}_{\widetilde{\mathcal{B}}} \left[\sup_{f \in \mathcal{V}} \mathcal{L}(\mathcal{T}_K f) - 3 \mathcal{L}_m (\mathcal{T}_K f; \widetilde{\mathcal{B}})\right] \leq \frac{16 K^2}{m} \mathbb{E}_{\widetilde{\mathcal{B}}} \log \mathcal{N}(\delta, \mathcal{H}, d_{\widetilde{\mathcal{B}}, 1}) + 8 \delta.
\end{align*}
It remains to bound $\mathcal{N}(\delta, \mathcal{H}, d_{\widetilde{\mathcal{B}}, 1})$ by the covering number of $\mathcal{V}$. Consider any $h_1, h_2 \in \mathcal{H}$ with $h_1=(\mathcal{T}_K f_1-f^*)^2$ and $h_2=(\mathcal{T}_K f_2-f^*)^2$, where $f_1, f_2 \in \mathcal{V}$. Then
\begin{align*}
\begin{aligned}
d_{\widetilde{\mathcal{B}}, 1} (h_1, h_2) & =\frac{1}{m} \sum_{i=1}^m |h_1(\widetilde{B}_i)-h_2(\widetilde{B}_i)| \\
& =\frac{1}{m} \sum_{i=1}^m |\mathcal{T}_K f_1(\widetilde{B}_i)-\mathcal{T}_K f_2(\widetilde{B}_i)|\cdot |\mathcal{T}_K f_1(\widetilde{B}_i)+\mathcal{T}_K f_2(\widetilde{B}_i)-2 f^*(\widetilde{B}_i)| \\
& \leq 4 K d_{\widetilde{\mathcal{B}}, 1} (\mathcal{T}_K f_1, \mathcal{T}_K f_2),
\end{aligned}
\end{align*}
which implies $\mathcal{N}(\delta, \mathcal{H}, d_{\widetilde{\mathcal{B}}, 1}) \leq \mathcal{N}(\delta /(4 K), \mathcal{T}_K \mathcal{V}, d_{\widetilde{\mathcal{B}}, 1})$. As a result,
\begin{align*}
\mathbb{E}_{\widetilde{\mathcal{B}}} \left[\sup_{f \in \mathcal{V}} \mathcal{L}(\mathcal{T}_K f) - 3 \mathcal{L}_m (\mathcal{T}_K f; \widetilde{\mathcal{B}})\right] & \leq \frac{16 K^2}{m} \mathbb{E}_{\widetilde{\mathcal{B}}} \log \mathcal{N}(\delta, \mathcal{T}_K \mathcal{V}, d_{\widetilde{\mathcal{B}}, 1}) + 32 K \delta \\
& \leq \frac{16 K^2}{m} \sup_{|\widetilde{\mathcal{B}}|=m} \log \mathcal{N}(\delta, \mathcal{T}_K \mathcal{V}, d_{\widetilde{\mathcal{B}}, 1}) + 32 K \delta,
\end{align*}
which completes the proof of (\ref{eq: 16}).

$\bullet$ Step 2: We show that
\begin{align}
\label{eq: 22}
\begin{aligned}
\mathbb{E}_{\mathcal{B}} \left[\sup_{f\in\mathcal{V}} \mathcal{L}(\mathcal{T}_K f) - 6 \mathcal{L}_{\tilde{n}} (\mathcal{T}_K f; \mathcal{B})\right] & \leq \frac{16 K^2}{\lfloor\frac{\tilde{n}}{l}\rfloor} \sup_{|\widetilde{\mathcal{B}}|=\lfloor\frac{\tilde{n}}{l}\rfloor} \log \mathcal{N}(\delta, \mathcal{T}_K \mathcal{V}, d_{\widetilde{\mathcal{B}}, 1}) + 32 K \delta \\
&~~~ + 8 K^2 \left\lfloor \tilde{n}/l \right\rfloor \beta((l-1)N+1).
\end{aligned}
\end{align}

\textit{Proof of step 2.}
We have
\begin{align*}
& \mathbb{E}_{\mathcal{B}} \left[\sup_{f\in\mathcal{V}} \mathcal{L}(\mathcal{T}_K f) - 6 \mathcal{L}_{\tilde{n}} (\mathcal{T}_K f; \mathcal{B})\right] \\
& = \mathbb{E}_{\mathcal{B}}\left[\sup_{f\in\mathcal{V}} \mathbb{E}[(\mathcal{T}_K f-f^*)^2] - \frac{6}{\tilde{n}} \sum_{i=1}^{\tilde{n}} (\mathcal{T}_K f(B_i)-f^*(B_i))^2 \right] \\
& \leq \mathbb{E}_{\mathcal{B}} \left[\sup_{f\in\mathcal{V}} \mathbb{E}[(\mathcal{T}_K f-f^*)^2] - \frac{3}{l \lfloor\frac{\tilde{n}}{l}\rfloor} \sum_{i=1}^{\tilde{n}} (\mathcal{T}_K f(B_i)-f^*(B_i))^2 \right] \\
& \leq \mathbb{E}_{\mathcal{B}} \left[\sup_{f\in\mathcal{V}} \mathbb{E}[(\mathcal{T}_K f-f^*)^2] - \frac{1}{l} \sum_{a=1}^l \frac{3}{\lfloor\frac{\tilde{n}}{l}\rfloor} \sum_{B \in \mathcal{B}^{(a)}} (\mathcal{T}_K f(B)-f^*(B))^2 \right] \\
& \leq \frac{1}{l} \sum_{a=1}^l \mathbb{E}_{\mathcal{B}} \left[ \sup_{f\in\mathcal{V}} \mathbb{E}[(\mathcal{T}_K f-f^*)^2] - \frac{3}{\lfloor\frac{\tilde{n}}{l}\rfloor} \sum_{B \in \mathcal{B}^{(a)}} (\mathcal{T}_K f(B)-f^*(B))^2 \right] \\
& = \frac{1}{l} \sum_{a=1}^l \mathbb{E}_{\mathcal{B}^{(a)}} \left[ \sup_{f\in\mathcal{V}} \mathcal{L}(\mathcal{T}_K f) - 3 \mathcal{L}_{\lfloor\frac{\tilde{n}}{l}\rfloor} (\mathcal{T}_K f; \mathcal{B}^{(a)}) \right],
\end{align*}
where in the second inequality we discard the tail term $\sum_{i=l \lfloor\frac{\tilde{n}}{l}\rfloor + 1}^{\tilde{n}} (\mathcal{T}_K f(B_i)-f^*(B_i))^2$, and in the third inequality we use the sub-additivity of the supremum. We next apply Proposition \ref{proposition: 2} to relate mixing and independent cases for each summand. Since $\sup_{f\in\mathcal{V}} (\mathcal{L}(\mathcal{T}_K f) - 3 \mathcal{L}_{\lfloor\frac{\tilde{n}}{l}\rfloor} (\mathcal{T}_K f; \mathcal{B}^{(a)})) \leq \sup_{f\in\mathcal{V}} \mathcal{L}(\mathcal{T}_K f) \leq 4 K^2$, Proposition \ref{proposition: 2} yields
\begin{align*}
& \mathbb{E}_{\mathcal{B}^{(a)}} \left[ \sup_{f\in\mathcal{V}} \mathcal{L}(\mathcal{T}_K f) - 3 \mathcal{L}_{\lfloor\frac{\tilde{n}}{l}\rfloor} (\mathcal{T}_K f; \mathcal{B}^{(a)}) \right] \\
& \leq \mathbb{E}_{\widetilde{\mathcal{B}}} \left[ \sup_{f\in\mathcal{V}} \mathcal{L}(\mathcal{T}_K f) - 3 \mathcal{L}_{\lfloor\frac{\tilde{n}}{l}\rfloor} (\mathcal{T}_K f; \widetilde{\mathcal{B}}) \right] + 8 K^2 \left\lfloor\frac{\tilde{n}}{l}\right\rfloor \beta((l-1)N+1),
\end{align*}
which implies
\begin{align*}
\mathbb{E}_{\mathcal{B}} \left[\sup_{f\in\mathcal{V}} \mathcal{L}(\mathcal{T}_K f) - 6 \mathcal{L}_{\tilde{n}} (\mathcal{T}_K f; \mathcal{B})\right] & \leq \frac{1}{l} \sum_{a=1}^l \mathbb{E}_{\mathcal{B}^{(a)}} \left[ \sup_{f\in\mathcal{V}} \mathcal{L}(\mathcal{T}_K f) - 3 \mathcal{L}_{\lfloor\frac{\tilde{n}}{l}\rfloor} (\mathcal{T}_K f; \mathcal{B}^{(a)}) \right] \\
& \leq \mathbb{E}_{\widetilde{\mathcal{B}}} \left[ \sup_{f\in\mathcal{V}} \mathcal{L}(\mathcal{T}_K f) - 3 \mathcal{L}_{\lfloor\frac{\tilde{n}}{l}\rfloor} (\mathcal{T}_K f; \widetilde{\mathcal{B}}) \right] + 8 K^2 \left\lfloor\frac{\tilde{n}}{l}\right\rfloor \beta((l-1)N+1).
\end{align*}
Applying (\ref{eq: 16}) with $m=\lfloor\frac{\tilde{n}}{l}\rfloor$, we have
\begin{align*}
& \mathbb{E}_{\mathcal{B}} \left[\sup_{f\in\mathcal{V}} \mathcal{L}(\mathcal{T}_K f) - 6 \mathcal{L}_{\tilde{n}} (\mathcal{T}_K f; \mathcal{B})\right] \\
& \leq \mathbb{E}_{\widetilde{\mathcal{B}}} \left[ \sup_{f\in\mathcal{V}} \mathcal{L}(\mathcal{T}_K f) - 3 \mathcal{L}_{\lfloor\frac{\tilde{n}}{l}\rfloor} (\mathcal{T}_K f; \widetilde{\mathcal{B}}) \right] + 8 K^2 \left\lfloor\frac{\tilde{n}}{l}\right\rfloor \beta((l-1)N+1) \\
& \leq \frac{16 K^2}{\lfloor\frac{\tilde{n}}{l}\rfloor} \sup_{|\widetilde{\mathcal{B}}|=\lfloor\frac{\tilde{n}}{l}\rfloor} \log \mathcal{N}(\delta, \mathcal{T}_K \mathcal{V}, d_{\widetilde{\mathcal{B}}, 1}) + 32 K \delta + 8 K^2 \left\lfloor\frac{\tilde{n}}{l}\right\rfloor \beta((l-1)N+1),
\end{align*}
which completes the proof of (\ref{eq: 22}).

$\bullet$ Step 3: With $\hat{f}$ defined in (\ref{eq: 20}), we have
\begin{align}
\label{eq: 19}
\begin{aligned}
\mathbb{E}[\mathcal{L}_{n-N+1} (\mathcal{T}_K \hat{f}; \mathcal{S})] & \leq \frac{8\sigma^2}{n - N + 1} \left(\sup_{\mathcal{S}} \log \mathcal{N}(\delta, \mathcal{T}_K \mathcal{V}, d_{\mathcal{S},2}) + 2\right) \\
&~~~ + \left(12 + 4 \sqrt{2 \log (2K\delta^{-1})}\right) \sigma \delta + 2 \inf_{f \in \mathcal{V}} \mathbb{E}[(\mathcal{T}_K f - f^*)^2],
\end{aligned}
\end{align}
where
\begin{align*}
d_{\mathcal{S},2}(f,g) := \left(\frac{1}{n-N+1} \sum_{t=N}^n (f(S_t)-g(S_t))^2\right)^{1/2}.
\end{align*}

\textit{Proof of step 3.}
For any $f \in \mathcal{V}$, we have
\begin{align*}
& \mathbb{E}_{\mathcal{X}}[\mathcal{L}_{n-N+1} (\mathcal{T}_K f; \mathcal{S})] \\
& = \mathbb{E}_{\mathcal{X}}\left[\frac{1}{n-N+1} \sum_{t=N}^n (\mathcal{T}_K f(x_{t-N+1}, \ldots, x_t) - f^*(x_{t-N+1}, \ldots, x_t))^2\right] \\
& = \frac{1}{n-N+1} \sum_{t=N}^n \mathbb{E}_{\mathcal{X}}[(\mathcal{T}_K f(x_{t-N+1}, \ldots, x_t) - f^*(x_{t-N+1}, \ldots, x_t))^2] \\
& = \frac{1}{n-N+1} \sum_{t=N}^n \mathbb{E}_{(x_{t-N+1}, \ldots, x_t)}[(\mathcal{T}_K f(x_{t-N+1}, \ldots, x_t) - f^*(x_{t-N+1}, \ldots, x_t))^2] \\
& = \mathbb{E}[(\mathcal{T}_K f - f^*)^2],
\end{align*}
where the third equality follows from marginalization, and the fourth equality holds because $(x_{t-N+1}, \ldots, x_t),\, t=N,\ldots,n$ have the same distribution due to the assumption of stationarity, which implies that
\begin{align*}
& \mathbb{E}_{\mathcal{X},\varepsilon} [\mathcal{L}_{n-N+1} (\mathcal{T}_K \hat{f}; \mathcal{S})] \\
& \leq \mathbb{E}_{\mathcal{X}} [\mathcal{L}_{n-N+1} (\mathcal{T}_K f; \mathcal{S})] + 2 \mathbb{E}_{\mathcal{X},\varepsilon} \left[ \frac{1}{n-N+1} \sum_{t=N}^{n} \varepsilon_t (\mathcal{T}_K \hat{f}(S_t) - \mathcal{T}_K f(S_t)) \right]  \\
&= \mathbb{E}[(\mathcal{T}_K f-f^*)^2] + 2 \mathbb{E}_{\mathcal{X},\varepsilon} \left[ \frac{1}{n-N+1} \sum_{t=N}^{n} \varepsilon_t \mathcal{T}_K \hat{f}(S_t) \right],
\end{align*}
where in the inequality we use the fact that $\hat{f}$ is an empirical risk minimizer and the equality follows from $\mathbb{E}_{\mathcal{X},\varepsilon} [\sum_{t=N}^{n} \varepsilon_t \mathcal{T}_K f(S_t)] = 0$. Taking infimum over all $f \in \mathcal{V}$, we have
\begin{align}\label{eq: 23}
\mathbb{E}_{\mathcal{X},\varepsilon} [\mathcal{L}_{n-N+1} (\mathcal{T}_K \hat{f}; \mathcal{S})] \leq 2 \mathbb{E}_{\mathcal{X},\varepsilon} \left[ \frac{1}{n-N+1} \sum_{t=N}^{n} \varepsilon_t \mathcal{T}_K \hat{f}(S_t) \right] + \inf_{f \in \mathcal{V}} \mathbb{E} \left[ (\mathcal{T}_K f - f^*)^2 \right].
\end{align}

For any fixed observations $\mathcal{X}=\{x_t\}_{t=1}^n$, let $\mathcal{V}_{\delta}(\mathcal{X})$ be a $\delta$-covering of $\mathcal{T}_K \mathcal{V}$ with respect to the pseudo-metric $d_{\mathcal{S},2}$ with minimal cardinality $\mathcal{N}(\delta, \mathcal{T}_K \mathcal{V}, d_{\mathcal{S},2})$. By definition of covering number, there exists $g \in \mathcal{V}_{\delta}(\mathcal{X})$ such that $d_{\mathcal{S},2}(\mathcal{T}_K \hat{f}, g) \leq \delta$. Then by Cauchy–Schwarz inequality,
\begin{align*}
\mathbb{E}_{\varepsilon} \left[ \frac{1}{n-N+1} \sum_{t=N}^{n} \varepsilon_t (\mathcal{T}_K \hat{f}(S_t)-g(S_t)) \right] & \leq d_{\mathcal{S},2}(\mathcal{T}_K \hat{f}, g) \, \mathbb{E}_{\varepsilon} \sqrt{\frac{1}{n-N+1} \sum_{t=N}^{n} \varepsilon_t^2} \\
& \leq \delta \sqrt{\frac{1}{n-N+1} \sum_{t=N}^{n} \mathbb{E}_{\varepsilon}[\varepsilon_t^2]} \leq \delta \sigma.
\end{align*}
Hence,
\begin{align*}
\mathbb{E}_{\varepsilon} \left[\frac{1}{n-N+1} \sum_{t=N}^{n} \varepsilon_t \mathcal{T}_K \hat{f}(S_t)\right] & = \mathbb{E}_{\varepsilon} \left[ \frac{1}{n-N+1} \sum_{t=N}^{n} \varepsilon_t (\mathcal{T}_K \hat{f}(S_t) - g(S_t) + g(S_t) - f^*(S_t)) \right] \\
&\leq \mathbb{E}_{\varepsilon} \left[ \frac{1}{n-N+1} \sum_{t=N}^{n} \varepsilon_t (g(S_t) - f^*(S_t)) \right] + \sigma \delta \\
&\leq \mathbb{E}_{\varepsilon} \left[ \frac{d_{\mathcal{S},2}(\mathcal{T}_K \hat{f}, f^*) + \delta}{\sqrt{n-N+1}} \cdot \frac{\sum_{t=N}^{n} \varepsilon_t (g(S_t) - f^*(S_t))}{\sqrt{n-N+1}\, d_{\mathcal{S},2}(g,f^*)} \right] + \sigma \delta.
\end{align*}
Note that conditioned on $\mathcal{X}$, the random variables 
\begin{align*}
Z_g := \frac{\sum_{t=N}^{n} \varepsilon_t (g(S_t) - f^*(S_t))}{\sqrt{n-N+1}\, d_{\mathcal{S},2}(g,f^*)}
\end{align*}
are sub-Gaussian. Using the union bound, we have
\begin{align*}
\mathbb{E}_{\varepsilon} \left[ \max_{g \in \mathcal{V}_{\delta}(\mathcal{X})} Z_g^2 \right] &= \int_{0}^{\infty} \mathbb{P} \left( \max_{g \in \mathcal{V}_{\delta}(\mathcal{X})} Z_g^2 > t \right) \dd t \\
& \leq T + \sum_{g \in \mathcal{V}_{\delta}(\mathcal{X})} \int_{T}^{\infty} \mathbb{P} \left( |Z_g| > \sqrt{t} \right) \dd t \\
&\leq T + |\mathcal{V}_{\delta}(\mathcal{X})| \int_{T}^{\infty} 2 e^{-t/(2\sigma^2)} \dd t \\
&= T + 4 \sigma^2 |\mathcal{V}_{\delta}(\mathcal{X})| e^{-T/(2\sigma^2)}.
\end{align*}
Taking $T = 2 \sigma^2 \log |\mathcal{V}_{\delta}(\mathcal{X})|$ in the above inequality gives that
\begin{align*}
\mathbb{E}_{\varepsilon} \left[ \max_{g \in \mathcal{V}_{\delta}(\mathcal{X})} Z_g^2 \right] \leq 2 \sigma^2 \left( \log \mathcal{N} (\delta, \mathcal{T}_K \mathcal{V}, d_{\mathcal{S},2}) + 2 \right).
\end{align*}
By Cauchy–Schwarz inequality, we obtain
\begin{align*}
&\mathbb{E}_{\mathcal{X},\varepsilon} \left[ \frac{1}{n-N+1} \sum_{t=N}^{n} \varepsilon_t \mathcal{T}_K \hat{f}(S_t) \right] \\
&\leq \frac{\sqrt{\mathbb{E}_{\mathcal{X},\varepsilon} [d_{\mathcal{S},2}^2 (\mathcal{T}_K \hat{f},f^*)]} + \delta}{\sqrt{n-N+1}} \sqrt{\mathbb{E}_{\mathcal{X},\varepsilon} \left[ Z_g^2 \right]} + \sigma \delta \\
&\leq \sigma \frac{ \sqrt{\mathbb{E}_{\mathcal{X},\varepsilon} [\mathcal{L}_{n-N+1} (\mathcal{T}_K \hat{f}; \mathcal{S})]} + \delta}{\sqrt{n-N+1}} \,  \sqrt{\mathbb{E}_{\mathcal{X}}[2 \log \mathcal{N} \left(\delta, \mathcal{T}_K \mathcal{V}, d_{\mathcal{S},2} \right) + 4]} + \sigma \delta.
\end{align*}
The inequality $\mathcal{N} (\delta, \mathcal{T}_K \mathcal{V}, d_{\mathcal{S},2}) \leq \left(2K/\delta\right)^{n-N+1}$ yields that
\begin{align*}
& \mathbb{E}_{\mathcal{X},\varepsilon} \left[ \frac{1}{n-N+1} \sum_{t=N}^{n} \varepsilon_t \mathcal{T}_K \hat{f}(S_t) \right] \\
& \leq \sigma \sqrt{\mathbb{E}_{\mathcal{X},\varepsilon}[\mathcal{L}_{n-N+1} (\mathcal{T}_K \hat{f}; \mathcal{S})]}\,\sqrt{\mathbb{E}_{\mathcal{X}}\left[\frac{2 \log \mathcal{N} \left( \delta, \mathcal{T}_K \mathcal{V}, d_{\mathcal{S},2} \right) + 4}{n-N+1}\right]} + \left(1+\sqrt{2\log(2K\delta^{-1})+4}\right) \sigma \delta.
\end{align*}
Substituting this into (\ref{eq: 23}), we obtain
\begin{align*}
\mathbb{E}_{\mathcal{X},\varepsilon} [\mathcal{L}_{n-N+1} (\mathcal{T}_K \hat{f}; \mathcal{S})] 
& \leq 2\sigma \sqrt{\mathbb{E}_{\mathcal{X},\varepsilon} [\mathcal{L}_{n-N+1} (\mathcal{T}_K \hat{f}; \mathcal{S})]} \, \sqrt{\mathbb{E}_\mathcal{X} \left[ \frac{2 \log \mathcal{N}(\delta, \mathcal{T}_K \mathcal{V}, d_{\mathcal{S},2}) + 4}{n - N + 1}\right]} \\
&~~~ + 2\left(1 + \sqrt{2 \log (2K\delta^{-1}) + 4}\right) \sigma \delta + \inf_{f \in \mathcal{V}} \mathbb{E}[(\mathcal{T}_K f - f^*)^2] \\
& \leq \frac{1}{2} \mathbb{E}_{\mathcal{X},\varepsilon} [\mathcal{L}_{n-N+1} (\mathcal{T}_K \hat{f}; \mathcal{S})] + 4 \sigma^2 \mathbb{E}_\mathcal{X} \left[\frac{\log \mathcal{N}(\delta, \mathcal{T}_K \mathcal{V}, d_{\mathcal{S},2}) + 2}{n - N + 1}\right] \\
&~~~ + 2\left(1 + \sqrt{2 \log (2K\delta^{-1}) + 4}\right) \sigma \delta + \inf_{f \in \mathcal{V}} \mathbb{E}[(\mathcal{T}_K f - f^*)^2],
\end{align*}
where we used the inequality $2 a b \leq a^2 / 2+2 b^2$. As a result,
\begin{align*}
& \mathbb{E}[\mathcal{L}_{n-N+1} (\mathcal{T}_K \hat{f}; \mathcal{S})] \\
& \leq 8\sigma^2 \mathbb{E}_\mathcal{X} \left[\frac{\log \mathcal{N}(\delta, \mathcal{T}_K \mathcal{V}, d_{\mathcal{S},2}) + 2}{n - N + 1}\right] + 4\left(1 + \sqrt{2 \log (2K\delta^{-1}) + 4}\right) \sigma \delta + 2 \inf_{f \in \mathcal{V}} \mathbb{E}[(\mathcal{T}_K f - f^*)^2] \\
& \leq \frac{8\sigma^2}{n - N + 1} \left(\sup_{\mathcal{S}} \log \mathcal{N}(\delta, \mathcal{T}_K \mathcal{V}, d_{\mathcal{S},2}) + 2\right) + \left(12 + 4 \sqrt{2 \log (2K\delta^{-1})}\right) \sigma \delta + 2 \inf_{f \in \mathcal{V}} \mathbb{E}[(\mathcal{T}_K f - f^*)^2],
\end{align*}
which completes the proof of (\ref{eq: 19}).

$\bullet$ Step 4: We decompose the error as
\begin{align}
\label{eq: 21}
\mathbb{E}[\mathcal{L}(\mathcal{T}_K \hat{f})] \leq \mathbb{E}[\mathcal{L}(\mathcal{T}_K \hat{f}) - 6 \mathcal{L}_{\tilde{n}} (\mathcal{T}_K \hat{f}; \mathcal{B})] + 12 N \mathbb{E}[\mathcal{L}_{n-N+1} (\mathcal{T}_K \hat{f}; \mathcal{S})].
\end{align}

\textit{Proof of step 4.}
We decompose the inference loss as
\begin{align*}
\mathbb{E}[\mathcal{L}(\mathcal{T}_K \hat{f})] = \mathbb{E}[\mathcal{L}(\mathcal{T}_K \hat{f}) - 12 N \mathcal{L}_{n-N+1} (\mathcal{T}_K \hat{f}; \mathcal{S})] + 12 N \mathbb{E}[\mathcal{L}_{n-N+1} (\mathcal{T}_K \hat{f}; \mathcal{S})],
\end{align*}
where the expectation is taken over the random sample $\{(x_t,y_t)\}_{t=1}^n$. Observe that
\begin{align*}
&\mathcal{L}(\mathcal{T}_K \hat{f}) - 12 N \mathcal{L}_{n-N+1} (\mathcal{T}_K \hat{f}; \mathcal{S}) \\
&= \mathbb{E}[(\mathcal{T}_K \hat{f}-f^*)^2] - \frac{12 N}{n-N+1} \sum_{t=N}^n (\mathcal{T}_K \hat{f}(x_{t-N+1}, \ldots, x_t) - f^*(x_{t-N+1}, \ldots, x_t))^2 \\
&\leq \mathbb{E}[(\mathcal{T}_K \hat{f}-f^*)^2] - \frac{12 N}{n-N+1} \sum_{\substack{t=lN \\ l \in \mathbb{N}, lN\leq n}} (\mathcal{T}_K \hat{f}(x_{t-N+1}, \ldots, x_t) - f^*(x_{t-N+1}, \ldots, x_t))^2 \\
&= \mathbb{E}[(\mathcal{T}_K \hat{f}-f^*)^2] - \frac{\tilde{n} (12 N)}{6(n-N+1)} \cdot \frac{6}{\tilde{n}} \sum_{\substack{t=lN \\ l \in \mathbb{N}, lN\leq n}} (\mathcal{T}_K \hat{f}(x_{t-N+1}, \ldots, x_t) - f^*(x_{t-N+1}, \ldots, x_t))^2 \\
&\leq \mathbb{E}[(\mathcal{T}_K \hat{f}-f^*)^2] - \frac{6}{\tilde{n}} \sum_{l=1}^{\tilde{n}} (\mathcal{T}_K \hat{f}(x_{(l-1)N+1}, \ldots, x_{lN}) - f^*(x_{(l-1)N+1}, \ldots, x_{lN}))^2 \\
&= \mathcal{L}(\mathcal{T}_K \hat{f}) - 6 \mathcal{L}_{\tilde{n}} (\mathcal{T}_K \hat{f}; \mathcal{B}),
\end{align*}
where 
\begin{align*}
\tilde{n} = |\{l \in \mathbb{N}: lN\leq n\}| = \left\lfloor n/N \right\rfloor
\end{align*}
is the number of terms in the summation, and we used $\frac{\tilde{n} (12 N)}{6(n-N+1)} \geq 1$ in the last inequality, which completes the proof of (\ref{eq: 21}).

$\bullet$ Step 5: We prove the claim.

\textit{Proof of Step 5.}
Combining (\ref{eq: 22}), (\ref{eq: 19}) and (\ref{eq: 21}), we obtain
\begin{align}\label{eq: 24}
\mathbb{E}[\mathcal{L}(\mathcal{T}_K \hat{f})] & \leq \mathbb{E}[\mathcal{L}(\mathcal{T}_K \hat{f}) - 6 \mathcal{L}_{\tilde{n}} (\mathcal{T}_K \hat{f}; \mathcal{B})] + 12 N \mathbb{E}[\mathcal{L}_{n-N+1} (\mathcal{T}_K \hat{f}; \mathcal{S})] \\
& \leq \mathbb{E}_{\mathcal{B}} \left[\sup_{f\in\mathcal{V}} \mathcal{L}(\mathcal{T}_K f) - 6 \mathcal{L}_{\tilde{n}} (\mathcal{T}_K f; \mathcal{B})\right] + 12 N \mathbb{E}[\mathcal{L}_{n-N+1} (\mathcal{T}_K \hat{f}; \mathcal{S})] \\
& \leq \frac{16 K^2}{\lfloor\frac{\tilde{n}}{l}\rfloor} \sup_{|\widetilde{\mathcal{B}}|=\lfloor\frac{\tilde{n}}{l}\rfloor} \log \mathcal{N}(\delta, \mathcal{T}_K \mathcal{V}, d_{\widetilde{\mathcal{B}}, 1}) + 32 K \delta + 8 K^2 \left\lfloor\frac{\tilde{n}}{l}\right\rfloor \beta((l-1)N+1) \\
&~~~ + \frac{96 N \sigma^2}{n - N + 1} \left(\sup_{\mathcal{S}} \log \mathcal{N}(\delta, \mathcal{T}_K \mathcal{V}, d_{\mathcal{S},2}) + 2\right) \\
&~~~ + \left(144 + 48 \sqrt{2 \log (2K\delta^{-1})}\right) N \sigma \delta + 24 N \inf_{f \in \mathcal{V}} \mathbb{E}[(\mathcal{T}_K f - f^*)^2].
\end{align}
Since $d_{\mathcal{S},2} \leq d_{\mathcal{Z}, \infty}$ and $d_{\widetilde{\mathcal{B}}, 1} \leq d_{\mathcal{Z}, \infty}$ when $\mathcal{Z}$ is taken to be $\mathcal{S}$ and $\widetilde{\mathcal{B}}$, respectively, and since $d_{\mathcal{Z}_1, \infty} \leq d_{\mathcal{Z}_2, \infty}$ whenever $\mathcal{Z}_1 \subseteq \mathcal{Z}_2$, it follows by definition that the corresponding covering numbers satisfy
\begin{align*}
\begin{gathered}
\sup_{|\widetilde{\mathcal{B}}|=\lfloor\frac{\tilde{n}}{l}\rfloor} \log \mathcal{N}(\delta, \mathcal{T}_K \mathcal{V}, d_{\widetilde{\mathcal{B}}, 1}) \leq \sup_{\mathcal{Z}: |\mathcal{Z}|=n} \log \mathcal{N}(\delta, \mathcal{T}_K \mathcal{V}, d_{\mathcal{Z}, \infty}), \\
\sup_{\mathcal{S}} \log \mathcal{N}(\delta, \mathcal{T}_K \mathcal{V}, d_{\mathcal{S},2}) \leq \sup_{\mathcal{Z}: |\mathcal{Z}|=n} \log \mathcal{N}(\delta, \mathcal{T}_K \mathcal{V}, d_{\mathcal{Z}, \infty}).
\end{gathered}
\end{align*}
For $n \geq 4 N l$, by choosing $\delta = 1/n$ and using the fact that $x/2 \leq \lfloor x \rfloor \leq x$ for $x \geq 1$, we can simplify (\ref{eq: 24}) as
\begin{align*}
\mathbb{E}[\mathcal{L}(\mathcal{T}_K \hat{f})] & \leq \frac{64 K^2 N l}{n} \sup_{\mathcal{Z}: |\mathcal{Z}|=n} \log \mathcal{N}(1/n, \mathcal{T}_K \mathcal{V}, d_{\mathcal{Z}, \infty}) + \frac{32 K}{n} + \frac{8 K^2 n}{N l} \beta((l-1)N+1) \\
&~~~ + \frac{192 N \sigma^2}{n} \sup_{\mathcal{Z}: |\mathcal{Z}|=n} \log \mathcal{N}(1/n, \mathcal{T}_K \mathcal{V}, d_{\mathcal{Z}, \infty}) + \frac{384 N \sigma^2}{n} \\
&~~~ + \frac{144 N \sigma + 96 N K \sigma \sqrt{\log n}}{n} + 24 N \inf_{f \in \mathcal{V}} \mathbb{E}[(\mathcal{T}_K f - f^*)^2] \\
& = 24 N \inf_{f \in \mathcal{V}} \mathbb{E}[(\mathcal{T}_K f - f^*)^2] + \frac{192 N \sigma^2 + 64 K^2 N l}{n} \sup_{\mathcal{Z}: |\mathcal{Z}|=n} \log \mathcal{N}(1/n, \mathcal{T}_K \mathcal{V}, d_{\mathcal{Z}, \infty}) \\
&~~~ \frac{32 K + 384 N \sigma^2 + 144 N \sigma + 96 N K \sigma \sqrt{\log n}}{n} + \frac{8 K^2 n}{N l} \beta((l-1)N+1),
\end{align*}
which completes the proof.
\end{proof}

\begin{lemma}[\cite{yu1994rates}, Corollary 2.7]
\label{lemma: 9}
Let $m \geq 1$ and suppose that $h$ is a measurable function, with absolute value bounded by $K$, on a product probability space $(\prod_{j=1}^m \Omega_j, \prod_{i=1}^m \sigma_{r_i}^{s_i})$ where $r_i \leq s_i \leq r_{i+1}$ for all $i$. Let $Q$ be a probability measure on the product space with marginal measures $Q_i$ on $(\Omega_i, \sigma_{r_i}^{s_i})$, and let $Q^{i+1}$ be the marginal measure of $Q$ on $(\prod_{j=1}^{i+1} \Omega_j, \prod_{j=1}^{i+1} \sigma_{r_j}^{s_j}), i=1, \ldots, m-1$. Let $\beta(Q)=\sup_{1 \leq i \leq m-1} \beta(k_i)$, where $k_i=r_{i+1}-s_i$, and $P=\prod_{i=1}^m Q_i$. Then,
\begin{align*}
|\mathbb{E}_{Q}[h]-\mathbb{E}_{P}[h]| \leq 2 (m-1) K \beta(Q).
\end{align*}
\end{lemma}
The lemma quantifies the discrepancy between the distributions of $m$ blocks when the blocks are independent and when they are dependent. The distribution within each block is identical in both scenarios. For a monotonically decreasing function $\beta$, we have $\beta(Q)=\beta(k^*)$, where $k^*=\min_i k_i$ is the smallest gap between blocks.

\begin{proposition}[sub-block selection]
\label{proposition: 2}
Let $h$ be a real-valued Borel measurable function such that $|h| \leq K$ for $K \geq 0$. Assume $\mathcal{B}^{(a)}$ has $m$ dependent blocks. Then
\begin{align*}
|\mathbb{E} [h(\widetilde{\mathcal{B}})] - \mathbb{E}[h(\mathcal{B}^{(a)})]| \leq 2 K m \beta((l-1)N+1),
\end{align*}
where $\widetilde{\mathcal{B}}$ has $m$ independent blocks drawn from the same distribution $\Pi$ as in $\mathcal{B}^{(a)}$.
\end{proposition}

\begin{proof}
The proof is adapted from \cite[Proposition 2]{kuznetsov2017generalization}. Let $P_j$ denote the joint distribution of $B_a, B_{l+a}, \ldots, B_{(j-1)l+a}$ and let $Q_j$ denote the distribution of $B_{(j-1)l+a}$ (or equivalently $\widetilde{B}_j$). Set $P=P_m$ and $Q=Q_1 \otimes \ldots \otimes Q_m$. In other words, $P$ and $Q$ are distributions of $\mathcal{B}^{(a)}$ and $\widetilde{\mathcal{B}}$ respectively. Then
\begin{align*}
|\mathbb{E} [h(\widetilde{\mathcal{B}})]-\mathbb{E} [h(\mathcal{B}^{(a)})]| = |\mathbb{E}_Q[h] - \mathbb{E}_P[h]| \leq 2 (m-1) K \beta(Q) \leq 2 K m \beta((l-1)N+1)
\end{align*}
by Lemma \ref{lemma: 9} and the fact that the gap between any two consecutive blocks is $(l-1)N+1$.
\end{proof}

\bibliographystyle{plain}
\bibliography{reference}

\end{document}